\documentclass{article}
\usepackage{amsthm,amssymb,amsmath,amsfonts}
\usepackage{lmodern}
\usepackage{adjustbox}

\usepackage{algorithm}
\usepackage{algorithmic}

\usepackage{tikz}

\allowdisplaybreaks

\let\originalleft\left
\let\originalright\right
\renewcommand{\left}{\mathopen{}\mathclose\bgroup\originalleft}
\renewcommand{\right}{\aftergroup\egroup\originalright}

\newtheorem{assumption}{Assumption}
 \newtheorem{lemma}{Lemma}
 \newtheorem{theorem}{Theorem}
 \newtheorem*{theorem*}{Theorem}
 \newtheorem{definition}{Definition}

\definecolor{pear}{HTML}{c60404}
\definecolor{lightblue}{HTML}{2980B9}

\newcommand{\set}[1]{\left\{#1\right\}}
\newcommand{\pa}[1]{\left(#1\right)}
\newcommand{\abs}[1]{\left|#1\right|}
\newcommand{\norm}[1]{\left\|#1\right\|}
\newcommand{\imp}{\Rightarrow}

\newcommand{\sign}{\mathrm{sign}}

\DeclareMathOperator*{\argmax}{arg\,max}
\DeclareMathOperator*{\argmin}{arg\,min}
\newcommand{\ceil}[1]{\left\lceil#1\right\rceil}

\newcommand{\transpose}{^\mathsf{\scriptscriptstyle T}}
\newcommand{\bsl}{\backslash}
\newcommand{\reach}[1]{\overset{#1}{\rightsquigarrow}}

\newcommand{\bmu}{{\boldsymbol \mu}}
\newcommand{\btheta}{{\boldsymbol \theta}}

\newcommand{\blambda}{{\boldsymbol \lambda}}

\newcommand{\bheta}{{\boldsymbol \eta}}

\newcommand{\ba}{{\bf a}}

\newcommand{\bb}{{\bf b}}
\newcommand{\bB}{{\bf B}}
\newcommand{\bc}{{\bf c}}

\newcommand{\bp}{\boldsymbol{p}}

\newcommand{\be}{{\bf e}}

\newcommand{\bk}{{\bf k}}

\newcommand{\bs}{{\bf s}}
\newcommand{\bq}{{\bf q}}
\newcommand{\bu}{{\bf u}}

\newcommand{\bv}{{\bf v}}

\newcommand{\by}{{\bf y}}
\newcommand{\bx}{{\bf x}}
\newcommand{\bX}{{\bf X}}
\newcommand{\bY}{{\bf Y}}
\newcommand{\bZ}{{\bf Z}}

\newcommand{\R}{\mathbb{R}}
\newcommand{\N}{\mathbb{N}}

\newcommand{\cA}{\mathcal{A}}

\newcommand{\cE}{\mathcal{E}}
\newcommand{\cF}{\mathcal{F}}
\newcommand{\cG}{\mathcal{G}}
\newcommand{\cH}{\mathcal{H}}

\newcommand{\cM}{\mathcal{M}}
\newcommand{\cN}{\mathcal{N}}
\newcommand{\cO}{\mathcal{O}}
\newcommand{\cP}{\mathcal{P}}

\newcommand{\cS}{\mathcal{S}}

\newcommand{\cU}{\mathcal{U}}

\newcommand{\EE}[1]{\mathbb{E}\left[#1\right]}

\newcommand{\EEc}[2]{\mathbb{E}\left[\left.#1\right|#2\right]}

\newcommand{\PP}[1]{\mathbb{P}\left[#1\right]}

\newcommand{\Prb}{\mathbb{P}}

\newcommand{\PPc}[2]{\mathbb{P}\left[\left.#1\right|#2\right]}

\newcommand{\II}[1]{\mathbb{I}{\left\{#1\right\}}}

\newcommand{\Bernoulli}{\mathrm{Bernoulli}}

\newcommand{\mean}[1]{\bar\mu_{#1}}

\newcommand{\fA}{\mathfrak{A}}
\newcommand{\fB}{\mathfrak{B}}
\newcommand{\fC}{\mathfrak{C}}
\newcommand{\fD}{\mathfrak{D}}

\newcommand{\fM}{\mathfrak{M}}

\newcommand{\fR}{\mathfrak{R}}
\newcommand{\fS}{\mathfrak{S}}
\newcommand{\fT}{\mathfrak{T}}

\newcommand{\fY}{\mathfrak{Y}}
\newcommand{\fZ}{\mathfrak{Z}}

\newcommand{\eps}{\varepsilon}
\renewcommand{\tilde}{\widetilde}
\renewcommand{\bar}{\overline}

\newcommand{\counter}[1]{N_{#1}}
\newcommand{\ora}{\mathrm{Oracle}}

\makeatletter
\newcommand{\numrel}[2]{
  \refstepcounter{equation}
  \ltx@label{#2}
  \overset{(\theequation)}{#1}
}
\newcommand{\numterm}[1]{\refstepcounter{equation} \ltx@label{#1} (\theequation)}
\makeatother

\newcounter{mylabelcounter}

\makeatletter
\newcommand{\labelText}[2]{%
#1\refstepcounter{mylabelcounter}%
\immediate\write\@auxout{%
  \string\newlabel{#2}{{1}{\thepage}{{\unexpanded{#1}}}{mylabelcounter.\number\value{mylabelcounter}}{}}%
}%
}
\makeatother
\usepackage[colorinlistoftodos, textwidth=22mm, shadow]{todonotes}

     \usepackage[final]{neurips_2022}

\usepackage{hyperref}
\hypersetup{colorlinks,linkcolor={red},citecolor={blue},urlcolor={brown}} 

\usepackage[utf8]{inputenc} 
\usepackage[T1]{fontenc}    
\usepackage{hyperref}       
\usepackage{url}            
\usepackage{booktabs}       
\usepackage{amsfonts}       
\usepackage{nicefrac}       
\usepackage{microtype}      

\usepackage{changepage}     

\title{When Combinatorial Thompson Sampling meets Approximation Regret}

\author{%
  Pierre Perrault \\
  Idemia\\
  \texttt{pierre.perrault@idemia.com}
}

\begin{document}

\maketitle
\begin{abstract}

We study the Combinatorial Thompson Sampling policy (\textsc{cts}) for combinatorial multi-armed bandit problems (CMAB), within an \emph{approximation regret} setting. Although \textsc{cts} has attracted a lot of interest, it has a drawback that other
usual CMAB policies do not have when considering non-exact oracles: for some oracles, \textsc{cts} has a poor approximation regret (scaling linearly with the time horizon~$T$) \citep{Wang2018}. 
A study is then necessary to discriminate the oracles on which \textsc{cts} could learn.
This study was started by \citet{kong2021hardness}: they 
gave the first approximation regret analysis of \textsc{cts} for the \emph{greedy oracle}, obtaining an upper bound of order $\cO\pa{\log(T)/\Delta^2}$, where $\Delta$ is some minimal reward gap.
In this paper, our objective is to push this study further than the simple case of the greedy oracle.
We provide
the first $\cO\pa{\log(T)/\Delta}$ approximation regret upper bound for \textsc{cts}, obtained under a specific condition on the approximation oracle,
allowing a reduction
to the exact oracle analysis. We thus term this condition \textsc{reduce2exact}, and observe that it is
satisfied in many \emph{concrete} examples. Moreover, it can be extended to the \emph{probabilistically triggered arms} setting, 
thus capturing even more problems, such as \emph{online influence maximization}.
\end{abstract}

\section{Introduction}
\label{sec:intro}
Stochastic multi-armed bandits (MAB) \citep{robbins1952some,berry1985bandit,lai1985asymptotically} are decision-making problems in which an \emph{agent} acts sequentially in an uncertain environment. At each round $t\in \N^*$, the agent must select one arm from a fixed set of $n$ arms, denoted by $[n]\triangleq \set{1,\dots,n}$, using a \emph{policy}, based on the feedback from the previous rounds. Then it gets as feedback an \emph{outcome} $X_{i,t}\in \R$ --- a random variable sampled from $\Prb_{X_i}$, independently from previous rounds --- where $i$ is the selected arm and  $\Prb_{X_i}$ is a probability distribution --- unknown to the agent --- of mean $\mu_i^*$. 
The  goal for the agent  is to maximize the \emph{cumulative  reward}  over  a  total  of $T$ rounds  ($T$ is the \emph{time horizon} and may be  unknown). The  performance  metric of a policy is its \emph{regret} $R_T$,  which is  the expectation of the  difference over $T$ rounds between the cumulative reward of the policy that always picked the arm with the highest expected reward and the cumulative reward of the learning policy.  
MAB models the so called dilemma between exploration and exploitation, i.e.,  whether to continue exploring arms to obtain more information (and thus strengthen the confidence in the estimates of the distributions $\Prb_{X_i}$), or to use the information gathered by playing the best arm according to the observations so far.

In this paper, we study stochastic combinatorial multi-armed bandit (CMAB), with \emph{semi-bandit feedback}, a.k.a. stochastic semi-bandit (still abbreviated as CMAB in this paper), an extension of MAB where the agent plays an \emph{action} (also called \emph{super-arm}) $A_t\in \cA\subset \cP([n])$ at each round~$t$, where $\cA$ is fixed and called \emph{action space}. 
The feedback includes the outcomes of all base arms in the played super-arm.\footnote{Note that we will consider the \emph{probabilistically triggered arms} extension in this paper, i.e., where the feedback is on triggered arms. For brevity, we do not present this generalization in the introduction.} 
The expected reward, given $A_t$, is assumed to be in the form\footnote{Henceforth, we typeset vectors and matrices in bold and indicate components with indices, e.g., $\ba=(a_i)_{i\in [n]} \in \R^n$. We also let $\be_i$ be the $i^{th}$ canonical unit
vector of $\R^n$, and define the \emph{incidence vector} of any subset $A\subset [n]$ as \(\be_A\triangleq \sum_{i\in A}\be_i.\) We denote by $\ba\odot\bb\triangleq (a_ib_i)$ the Hadamard product of two vectors $\ba$ and~$\bb$.
} $r(A_t,\bmu^*)$, where $\bmu^*\in \R^n$ is the unknown vector of expectations (traditionally, the reward is linear and equal to $\be_{A_t}\transpose \bmu^*$).
In recent years, CMAB
has attracted a lot of interest (see e.g. \citet{cesa-bianchi2012combinatorial,gai2012combinatorial,chen13a,Chen2015combinatorial,kveton2015tight,wang2017improving,perrault:tel-03093268}), particularly due to its wide applications in network routing, online advertising, recommender system, influence marketing, etc.

Many CMAB policies are based on the \emph{Upper Confidence Bound} (UCB) approach, extending the classical \textsc{ucb} policy  \citep{auer2002finite} from MAB to CMAB. This type of approach uses an optimistic estimate $\bmu_t$ of $\bmu^*$ (i.e., for which the reward function is overestimated), lying in a well-chosen confidence region. Then, the action is chosen by plugging $\bmu_t$ inside an \emph{oracle} (typically, $\ora\pa{\bmu^*}$ is a maximizer of the reward function $A\mapsto r(A,\bmu^*)$).
 An example of such policy is \emph{Combinatorial Upper Confidence Bound} (\textsc{cucb}) \citep{chen13a,kveton2015tight}, that uses a Cartesian product of the individual confidence intervals of each arm as a confidence region. For mutually independent arms,
 \citet{combes2015combinatorial} provided the UCB-style policy  \emph{Efficient Sampling for Combinatorial Bandit} (\textsc{escb}), building a tighter axis-aligned ellipsoidal confidence region around the empirical mean, which helps to better restrict the exploration. \citet{Degenne2016} provided a policy called \textsc{ols-ucb}, leveraging a \emph{sub-Gaussianity} assumption on the arms to generalize the 
 \textsc{escb} approach.
 These policies have been further extended to more general settings afterwards \citep{perraultbudgeted2020, perrault2020covariance-adapting, perrault2019exploiting}.
 Although improving \textsc{cucb}, all these generalizations are  inefficient in terms of computation time.
 
Another paradigm that has recently gained interest (and which will be our focus in this paper) is to rely on \emph{Thompson Sampling} (\textsc{ts}) instead of \textsc{ucb}, still targetting \emph{frequentist regret}.
Although introduced much earlier by \citet{thompson1933likelihood}, the theoretical analysis of \textsc{ts} for MAB is quite recent: \citet{kaufmann2012thompson,agrawal2012thompsonarxiv} gave
a regret bound matching the \textsc{ucb} policy theoretically.  Moreover, \textsc{ts} often performs better than \textsc{ucb} in practice, making \textsc{ts} an attractive policy for further investigations. For CMAB, \textsc{ts} extends to \emph{Combinatorial Thompson Sampling}
(\textsc{cts}). In \textsc{cts}, the unknown mean $\bmu^*$ is associated with a belief (a prior distribution, that could be e.g. a product of Beta or Gaussian distributions) updated to a posterior with the Bayes’rule, each time a feedback is received. In order to choose an action at round~$t$, \textsc{cts} draws a sample $\btheta_t$ from the current belief, and plays the action given by $\ora\pa{\btheta_t}$.
\textsc{cts}
is an attractive policy because it has similar advantages to the previously mentioned policies working with ellipsoidal confidence regions while being, like \textsc{cucb},
computationally efficient. Indeed,
 recently, for mutually independent arms and sub-Gaussian arms respectively, \citet{Wang2018, perrault2020statistical} proposed tight analyses of \textsc{cts}. 

Unlike UCB-based policies, the analysis of \textsc{cts} is valid only when $\ora$ is exact, i.e., when \[\ora\pa{\bmu}\in \argmax_{A\in \cA} r(A,\bmu).\]
Although this holds true for many combinatorial problem described by the pair $(r,\cA)$ (we recall that $r$ is usually linear), there exist some problems where the requirement on $\ora$ has to be relaxed in order to make it tractable. 
This is usually done considering an $\alpha$-\emph{approximation} oracle \citep{chen13a, Chen2015combinatorial, wen2016influence}, for $\alpha\in(0,1)$:  
\begin{align}r(\ora\pa{\bmu},\bmu)\geq \alpha\max_{A\in \cA} r(A,\bmu).\label{eq:alphamax}\end{align}
Under an $\alpha$-approximation oracle, the benchmark cumulative reward is the $\alpha$-fraction of the optimal reward, leading to the notion of \emph{approximation regret} \citep{kakade2009playing,streeter2009online,Chen2015combinatorial}.

Aware of the limitation of their \textsc{cts} analysis (that works only with exact oracles), \citet{Wang2018} also proved, in their Theorem 2, that this limitation is not a technical artifact. More precisely, they provided a specific CMAB instance with an associated approximation oracle, such that \textsc{cts} on this instance and with this oracle must have a regret scaling linearly in $T$.
Although this negative result is of great interest to the research community, some concerns limit its consideration. Indeed, not only the CMAB instance provided by \citet{Wang2018} is actually a MAB one (meaning that there is an efficient oracle that simply enumerates the arms), so the use of an approximation regret is not justified, but above all, both the oracle and the instance are uncommon and designed for the proof.

Interested in the question of whether the example provided by \citet{Wang2018} is pathological or generalizable, \citet{kong2021hardness} recently
initiated a study, which revealed that linear approximation regret for \textsc{cts} seems to be pathological. More precisely, they derived a $\cO\pa{\log(T)/\Delta^2}$ bound for the specific case of a \emph{greedy} oracle\footnote{It is worth mentioning that this oracle is one of the most common, so it is logical to focus on it first.}, where $\Delta$ is some reward gap. 
This result is obtained by bounding the approximation regret by a \emph{greedy regret}, that simply replaces $\alpha\max_{A\in \cA} r(A,\bmu^*)$ with $r(\ora\pa{\bmu^*},\bmu^*)$, using equation~\eqref{eq:alphamax}. They also gave a tight lower bound on the greedy regret.

In this paper, we want to  explore another class of oracles covering more problems in practice. Our goal is to demonstrate that although there are instrumental examples of problems where \textsc{cts} has a linear regret, the majority of concrete problems do not follow this regime, and are in fact similar to the exact oracle case.


\paragraph{Contributions} 
With a specific condition on $\ora$, we describe a general set of CMAB
    problems where the approximation regret of \textsc{cts} has a $\cO\pa{\log(T)/\Delta}$ bound, improving by a factor $1/\Delta$ over the bound of \citet{kong2021hardness}. This does not contradict their lower bound, which is focused on the greedy regret.
    We call this set of CMAB
    problems \textsc{reduce2exact}, because, as we will see, a \textsc{reduce2exact} problem can be approximated using a reduction to sub-problems that can be solved exactly.
    Our main result is the approximation regret guaranty for \textsc{reduce2exact} problems, provided in Theorem~\ref{thm:tsapprox}.
    The \textsc{reduce2exact} condition on $\ora$ is structural and applies notably to greedy algorithms for submodular maximization.
    In particular, it allows to deal with problems such as \emph{probabilistic maximum coverage}.
    We note that \textsc{reduce2exact} is compatible with the \emph{probabilistically triggered arms} setting, which allows us to capture even more problems, such as \emph{online influence maximization} \citep{wen2017online,wang2017improving}.
    As we want to focus on concrete problems, we provide several other examples that belongs to \textsc{reduce2exact}: \emph{Metric k-center}, \emph{Vertex cover}, \emph{Max-Cut} and \emph{Travelling salesman problem}.

\paragraph{Further related work}
We refer the reader to \citet{Wang2018} for more related work on \textsc{ts} for combinatorial bandits. Briefly, one can mention \citet{gopalan2013thompson}, that gave a frequentist high-probability regret bounds for \textsc{ts} with
 a general action space and feedback model ---  
\citet{Komiyama2015}, that studied \textsc{ts} for the $m$-sets action space --- \citet{wen2015efficient}, that studied \textsc{ts} for
contextual CMAB problems, using the Bayesian regret metric (see also  \citet{russo2016information}).

\paragraph{Other known limitations of \textsc{cts}}
Apart from the limitation related to the approximation regret that interests us in this paper, there are some other existing limitations of \textsc{cts} highlighted in the literature, which we review here: The \textsc{cts} policy has an exponential constant term in its regret upper bound \citep{Wang2018,perrault2020statistical}, and \citet{Wang2018} proved in their Theorem~3 that this is unavoidable. 
A similar behavior have been demonstrated in \citet{zhang2021suboptimality}, where it is shown that  \textsc{cts} does
not scale polynomially in the ambient dimension $n$ in general. In addition, \citet{zhang2021suboptimality} also proved that  \textsc{cts} is not minimax optimal. Actually,
they even proved that in high dimensions, the minimax regret of \textsc{cts} is
almost linear in $T$.

\paragraph{The strengths of \textsc{cts}}
Despite the weaknesses mentioned above, \textsc{cts} remains a widely used policy, mainly because of its empirical performance. Indeed, \textsc{cts} generally outperforms other policies such as \textsc{cucb} and \textsc{escb} \citep{Wang2018, perrault2020statistical}. Moreover, it is relatively simple to implement, and is computationally efficient (just like \textsc{cucb}). On the theory side, another advantage is that for an exact oracle, \textsc{cts} is asymptotically quasi-optimal\footnote{This means that it has a distribution-dependent regret upper bound whose leading term in $T$ has an optimal rate, up to a poly-logarithmic factor in $n$.} for many settings where \textsc{cucb} is not, and where \textsc{escb} is computationally inefficient \citep{perrault2020statistical}. It would be desirable that these advantages also apply to the case of approximation regret, thus motivating our investigations.

\section{Model and definitions}\label{sec:model}

For more generality, we consider the \emph{probabilistically triggered arms} extension of CMAB \citep{Chen2015combinatorial,wang2017improving}, abbreviated to CMAB-T.
In this context, the action $A\in \cA$ selected is not necessarily equal to the triggered super-arm $S$. More precisely, the action
space $\cA$ is no longer necessarily a subset of $\cP([n])$ and can be  infinite. At round $t$, the agent selects $A_t\in \cA$, based on the history of observations $\cH_t\triangleq \sigma\pa{\bX_1\odot \be_{S_1},\dots,\bX_{t-1}\odot \be_{S_{t-1}}}$ and a possible extra source of randomness (we denote by $\cF_t$ the filtration containing $\cH_t$ and the extra randomness of round $t$ --- in particular, $A_t$ is $\cF_t$-measurable). Then, an independent sample $\bX_t\sim \Prb_{\bX},~\bX_t\in \R^n$ is drawn and a random subset $S_t\in \cS\subset \cP([n])$ of arms are triggered ($\cS$ is called \emph{super-arm space} or \emph{subset space}). We assume that $S_t$ is drawn independently from a distribution $D_{\text{trig}}\pa{A_t,\bX_t}$ and that the outcome of an arm does not depend on whether it is triggered. In addition, if we don't have $\Prb_{\bX}=\otimes_{i\in [n]}\Prb_{X_i}$, we assume that $D_{\text{trig}}$
doesn't depend on $\bX_t$. For the feedback, the outcome of each triggered arm is observed, i.e., $\be_{S_t}\odot\bX_t$ is observed. The expected reward is of the form $r(A_t,\bmu^*)$, where $r$ is a function defined on a domain $\cA\times \cM$, with $\cM\subset \R^n$. The objects $\cA,D_{\text{trig}},r$ are known to the agent.
We assume that $r(\cdot,\bmu^*)$ admits a maximum $r(A^*,\bmu^*)$ on $\cA$.
In the following, we give the definition of the probability that an arm $i\in [n]$ is triggered (and thus that a feedback from $i$ is obtained) by having played a certain action $A\in \cA$.
\begin{definition}[Triggering probabilities]\label{def:trig_prob}
The triggering probabilities are defined for all $i\in [n]$ and $A\in\cA$ as  \(p_i(A)\triangleq \PP{i\in S},\) where $S\sim D_{\text{trig}}\pa{A,\bX}$, $\bX \sim \Prb_{\bX}$.
\end{definition}
Under an $\alpha$-approximation $\ora$, we use the approximation regret to evaluate the performance of a policy $\pi$, defined as follows.
\begin{definition}[Approximation regret]\label{def:approx_regret}
The $T$-round $\alpha$-approximation regret of a learning policy $\pi$ that selects action $A_t\in \cA$ at round $t$ is defined as follows, where the approximation gap is defined as
\(\Delta_t=\Delta\pa{A_t}\triangleq 0\vee\pa{\alpha r(A^*,\bmu^*) - r(A_t,\bmu^*)},\)
with $A^*\in\argmax_{A\in \cA}r(A,\bmu^*)$. 
\[R_{T,\alpha}(\pi)\triangleq \EE{\sum_{t\in [T]}\Delta_t}.\]
\end{definition}
To approach the problem of minimizing $R_{T,\alpha}$, we consider the following standard assumptions \citep{wang2017improving}.

\begin{assumption}[Approximation oracle]\label{ass:approx} The agent has access to an $\ora$ such that for
any mean vector $\bmu\in \cM$,
\[r(\ora\pa{\bmu},\bmu)\geq \alpha r(A^*,\bmu).\]
\end{assumption}

\begin{assumption}[1-norm triggering probability modulated bounded smoothness]\label{ass:smooth} There exists $\bB\in \R_+^n$ such that for all $A\in \cA$, for all $\bmu,\bmu'\in \cM$, 
\[\abs{r(A,\bmu)-r(A,\bmu')}\leq\sum_{i\in [n]} p_i(A)B_i\abs{\mu_i-\mu'_i}. \]
\end{assumption}

\begin{assumption}[Sub-Gaussianity of the outcome distribution] \label{ass:subgau}
 $\Prb_{\bX}$ is such that $\forall \blambda\in\R^n$, \[\EE{e^{\blambda\transpose(\bX-\bmu^*)}}\leq e^{\norm{\blambda}_2^2/8}.\]
For example, $\Prb_{\bX}=\otimes_{i\in [n]}\Prb_{X_i}$, and $X_i\overset{a.s.}{\in} [0,1]$ (from Hoeffding’s Lemma \citep{hoeffding1963probability}).
\end{assumption}
\begin{definition}[Other definitions]
We define, for $i\in  [n]$, the minimal gap of an action containing $i$ as \[\Delta_{i,\min}\triangleq \inf_{A\in \cA:~p_i(A)>0,~\Delta(A)>0} \Delta(A).\] The minimal and maximal gaps are defined as \[\Delta_{\min}\triangleq\min_{i\in [n]}\Delta_{i,\min}\quad \text{and}\quad
\Delta_{\max}\triangleq\sup_{A\in \cA}\Delta(A).\] For $A\in \cA$, we let $\textsc{t}(A)\triangleq \set{i\in [n]: p_i(A)>0}$ be the set of arms that are triggerable by selecting action $A$. We finally define 
\[ m \triangleq\sup_{A\in \cA}\abs{\textsc{t}\pa{ A}},\quad
    m^*\triangleq \abs{\textsc{t}\pa{A^*}},\quad\text{and} \quad p^*\triangleq \inf_{i\in [n],~A\in \cA:~p_i(A)>0}p_i(A).
    \]
\end{definition}
\section{Combinatorial Thompson Sampling and exact oracle analysis}\label{sec:cts}
In this section, we present the \textsc{cts} policy, focusing on two versions, one working with a Beta prior, and the other with a Gaussian prior. Then, we present an 
 associated analysis for the exact oracle case (i.e., with $\alpha=1$).
\begin{algorithm}[t]
\begin{algorithmic}
\STATE \textbf{Initialization}: 
 For each arm $i$, let $\gamma_i=\delta_i=1$.
\STATE \textbf{For all} $t\geq 1$:
\STATE \quad Draw $\btheta_t\sim\otimes_{i\in [n]}\text{Beta}(\gamma_i,\delta_i)$. 
\STATE \quad Play $A_t=\ora\pa{\btheta_t}$.
\STATE \quad Get the observation $\bX_{t}\odot\be_{S_t}$, and draw $\bY_t\sim \otimes_{i\in S_t}\Bernoulli(X_{i,t})$.
\STATE \quad For all $i\in S_t$ update $\gamma_i\leftarrow \gamma_i+Y_{i,t}$ and $\delta_i\leftarrow \delta_i+1-Y_{i,t}$. 
\end{algorithmic}
\caption{\textsc{cts-beta}}\label{algo:tsbeta}
\end{algorithm}
\begin{algorithm}[t]
\begin{algorithmic}
\STATE \textbf{Input}: $\beta>1$.
\STATE \textbf{Initialization}: Play each arm once (if the agent knows that $\bmu^*\in [a,b]^n$, this might be skipped)
\STATE \textbf{For every subsequent round} $t$:
\STATE \quad Draw $\btheta_t\sim\otimes_{i\in [n]}\cN\pa{\bar\mu_{i,t-1},{ N^{-1}_{i,t-1}} \beta/4 }$ ($\theta_{i,t}\sim\cU[a,b]$ if $N_{i,t-1}=0$). \STATE \quad Play $A_t=\ora\pa{\btheta_t}$.
\STATE\quad Get the observations $\bX_{t}\odot\be_{S_t}$ and let $\bY_t=\bX_t$.
\STATE \quad Update $\bar\bmu_{t-1}$ and counters accordingly. 
\end{algorithmic}
\caption{\textsc{cts-gaussian}}\label{algo:tsgauss}
\end{algorithm}
\subsection{Algorithms}
Based on the above assumptions, we focus on two versions of \textsc{cts}. The first version is \textsc{cts-beta} \citep{Wang2018} (Algorithm~\ref{algo:tsbeta}), working when we assume furthermore that $\Prb_{\bX}=\otimes_{i\in [n]}\Prb_{X_i}$ and $\bX\overset{a.s.}{\in}[0,1]^n$ (note that this actually covers the case of bounded $\bX$, by adjusting the parameter $\bB$). For each arm $i\in [n]$, \textsc{cts-beta} maintains a Beta prior distribution with parameters $\gamma_i$ and $\delta_i$ (initialized to $1$). At each round $t$, for each arm $i$, the algorithm sample $\theta_{i,t}$ from the corresponding prior, representing the current estimate of $\mu_i^*$. Then the oracle outputs the action $A_t$ to play according to the input vector $\btheta_t$. Based on the observation feedback, the algorithm then updates
the corresponding Beta distributions.
The second version is \textsc{cts-gaussian}  \citep{perrault2020statistical} (Algorithm~\ref{algo:tsgauss}), that works under the more general Assumption~\ref{ass:subgau}. It is essentially the same as \textsc{cts-beta}, except that the prior distributions are Gaussian. 
For both Algorithm~\ref{algo:tsbeta} and Algorithm~\ref{algo:tsgauss}, and an arm $i\in [n]$, we define the number of time $i$ has been triggered at the beginning of round $t$, called counter of arm $i$, as 
 \[N_{i,t-1}\triangleq\sum_{t'\in[t-1]}\II{i\in S_{t'}}.\]
We also define the \emph{empirical mean} at the beginning of round $t$ as
\[\mean{i,t-1}\triangleq {\sum_{t'\in[t-1]}\frac{\II{i\in S_{t'}}Y_{i,t'}}{\counter{i,t-1}}}.\] 

\subsection{Analysis of \textsc{cts}: the $\alpha=1$ case}

Although this is close to some known results in the current literature \citep{huyuk2019analysis, perrault2020statistical}, there is no proof for the classical $\cO\pa{\log(m)\log(T)\sum_{i\in [n]} {B_i^2}/{\Delta_{i,\min}}}$ regret bound under the above assumptions in the CMAB-T setting, either for \textsc{cts-beta} (Algorithm~\ref{algo:tsbeta}) or \textsc{cts-gaussian} (Algorithm~\ref{algo:tsgauss}). We thus provide such a result in Theorem~\ref{thm:tsexact} (the proof is postponed to Appendix~\ref{app:tsexact}). We can notice a difference with the work of \citet{huyuk2019analysis} concerning the Assumption~\ref{ass:smooth}, where the triggering probabilities do not appear (and are present in the main term of their final regret bound). The main difference with \citet{perrault2020statistical} is that they do not consider probabilistically triggered arms. 
 \begin{theorem}\label{thm:tsexact}If $\Delta_{\min}>0$ and $p^*>0$, the policy $\pi$ described in
Algorithm~\ref{algo:tsbeta} (under Assumptions~\ref{ass:approx},~\ref{ass:smooth} and~$\Prb_{\bX}=\otimes_{i\in [n]}\Prb_{X_i}$, $\bX\overset{a.s.}{\in}[0,1]^n$) or Algorithm~\ref{algo:tsgauss} (under Assumptions~\ref{ass:approx},~\ref{ass:smooth}~and \ref{ass:subgau}) has a regret of order \[R_{T,1}(\pi)=\cO\pa{\sum_{i\in [n]} \frac{B_i^2\log(m)\log(T)}{\Delta_{i,\min}}}\cdot\]
 \end{theorem}
In addition to being a new result in itself, Theorem~\ref{thm:tsexact} will be useful for the $\alpha<1$ case.
Concerning the bound, after changing the CMAB setting to CMAB-T, it should be noticed that the $T$-independent additive constant depends on $1/p^*$, the cause being the use of the Lipschitz condition (Assumption~\ref{ass:smooth}) without weighting by probabilities. We think that this dependence should be avoidable, but it seems that another technique has to be considered. Finally, we remark that this kind of bound can be usually transformed into a gap-independent $\sqrt{T}$ bound \citep{chen13a}, however, for \textsc{cts}, achieving this transformation is impossible since
\textsc{cts} is not minimax optimal (as we mentioned in the paragraph "Other known limitations of CTS"). 

\section{The $\alpha<1$ case for \textsc{reduce2exact} problems}\label{sec:cond}
We will now look at the $\alpha<1$ case. Our strategy is based on the following observation: many approximation algorithms involve a relaxation, or a reduction to one or more problems that can be solved exactly (in this paper, we use the terminology sub-problems to refer to them). The approximation guarantee for such an approximation algorithm is thus obtained by linking the original problem to those sub-problems. 
We give a simple abstract example to illustrate this idea. Let's say we want to maximize a function $f$ on some set $F$, using an $\alpha$-approximation algorithm. 
Assume there exist two other functions $g$ and $h$ defined on sets $G$ and $H$ respectively, such that $G\times H\subset F$ and such that $g$ and $h$ can be maximized exactly on $G$ and $H$ respectively. Finally, assume that for all $(x,y)\in G\times H$, 
\begin{align}\alpha \max_{F}f-f(x,y)\leq \max_{G}g-g(x) + \max_{H}h-h(y).\label{rel:illua1}\end{align}
This means that an $\alpha$-approximation algorithm to maximize $f$ can simply output the feasible solution $(\argmax_Gg,\argmax_Hh)\in F$. This example may seem very basic and artifactual at first glance, but it turns out that many approximation algorithms rely on the same principle, as we will see in subsection~\ref{subsec:ex}.
To see how this can be exploited for approximation regret minimization, we can notice that the LHS of \eqref{rel:illua1} is an approximation gap (as defined in Definition~\ref{def:approx_regret}, taking the max with 0 and considering that the choice $(x,y)\in F$ is that of a policy at a given round, with $F$ playing the role of the action space) and that the RHS is the sum of two gaps, with $G$ (respectively $H$) playing the role of the action space, but this time without approximation factor. Summing over the rounds, we find that the corresponding approximation regret is bounded by the sum of two "classical" regrets (in this paper, we use the terminology sub-regrets). Thus, the bounds we obtain on these two sub-regrets using \textsc{cts} with the corresponding exact oracles translate into an bound on the approximation regret. 

\paragraph{Avoiding the "mismatch" phenomenon}
\citet{kong2021hardness} identified the reason why their regret bound has a $\Delta^2$ in the denominator, while the usual CMAB algorithms only have a $\Delta$. They term it a "mismatch" between the estimated
gaps that need to be eliminated by exploration and the actual regret the algorithm needs to pay. We argue that, in fact, this mismatch phenomenon should exist in principle, even for non-approximation regret. It is usually avoided using a smoothness assumption like Assumption~\ref{ass:smooth}, linking our estimations (here the arms that generate the feedback) and what is paid. In the above example, although we are in an approximation context, the situation is fundamentally no different. We see that \eqref{rel:illua1} links the paid approximation regret with two exact sub-regrets which are themselves assumed to be related to our estimates through a smoothness-like property. 
There is thus an indirect link between the approximation regret and the outcomes, which enables one of the $\Delta$ present in the denominator of the exploration term ${\log(T)}/\Delta^2$ to be cancelled out by the actual regret paid, thus avoiding the mismatch phenomenon.

To summarize, just as the approximation relation is obtained by linking the original problem to sub-problems that can be solved exactly, the idea behind \textsc{reduce2exact} problems is to link the approximation regret to several sub-regrets, each satisfying the appropriate properties for \textsc{cts}, namely a smoothness relationship and the availability of an exact oracle.
 We formalize this in the following assumption.
\begin{assumption}[\textsc{reduce2exact}]
\label{reduce2exact}
There exist $\ell\in \N^*$, $\bc\in \R_+^\ell$ and $\bB_j\in \R_+^n$ for all $j\in [\ell]$ such that the following is true.
 $\ora$ is of the form $\ora=\ora_2\circ\ora_1$, where $\ora_1$ and $\ora_2$ are described as follows.
 \begin{itemize}
     \item $\ora_1:$
     For $\bmu\in \cM$, $\ora_1(\bmu)$ must output a sequence $(E_1,\dots,E_{\ell})$ described as follows. For each $j\in [\ell]$, let
$\cE_j=\cE_j\pa{E_1,\dots,E_{j-1}}$ be a sub-action space which may depend on $E_1,\dots,E_{j-1}$ and let $r_j(\cdot,\bmu):\cE_j\to \R$ be a reward sub-function.
Then, we require that $E_j\in \argmax_{E\in \cE_j}r_j(E,\bmu)$.
\item $\ora_2:$ For an input $E_1\in \cE_1,\dots,E_\ell\in \cE_\ell$, $\ora_2\pa{E_1,\dots,E_\ell}$ must output an action in $\cA$ such that:
\begin{align}\Delta\pa{\ora_2\pa{E_1,\dots,E_\ell}}\leq
\sum_{j\in [\ell]}{ \pa{r_j\pa{E_j^*,\bmu^*}-r_j\pa{E_j,\bmu^*}}\cdot c_j},
\label{eq:ora2}\end{align}
where for all $j\in [\ell]$, $E_j^*\in \argmax_{E\in \cE_j}r_j(E,\bmu^*)$.
 \item Finally, in addition to the above constraints on $\ora_1$ and $\ora_2$, for each $j\in [\ell]$, we require that the reward sub-function $r_j$ satisfies Assumption~\ref{ass:smooth} with the constants $ \bB_j$ and with the triggering probabilities
\( p_i\pa{\ora_2\pa{E_1,\dots,E_\ell}},~i\in [n].\)
 \end{itemize}
\end{assumption}
 Informally, in the above Assumption~\ref{reduce2exact}, for $\bmu\in \cM$,
$\ora_1(\bmu)$ exactly solves a finite sequence of recursively defined optimization sub-problems and $\ora_2$ builds an action in $\cA$ for the original approximation problem using the intermediate solutions provided by $\ora_1$.
At first sight, Assumption~\ref{reduce2exact} seems very specific and rather difficult to fulfill. However, we will see in subsection~\ref{subsec:ex} that many concrete problems satisfy it.

\subsection{Analysis}
In this subsection, we give in Theorem~\ref{thm:tsapprox} the main result of this paper. It basically states that under Assumptions~\ref{reduce2exact}, the \textsc{cts} policy have a regret bound comparable to the exact oracle case. The proof is postponed to Appendix~\ref{app:tsapprox}.

 \begin{theorem}\label{thm:tsapprox}If $\Delta_{\min}>0$ and $p^*>0$, the policy $\pi$ described in
 Algorithm~\ref{algo:tsbeta} (under Assumption~\ref{reduce2exact} and~$\Prb_{\bX}=\otimes_{i\in [n]}\Prb_{X_i}$, $\bX\overset{a.s.}{\in}[0,1]^n$) or Algorithm~\ref{algo:tsgauss} (under Assumptions~\ref{ass:subgau} and~\ref{reduce2exact}) has regret of order \[R_{T,\alpha}(\pi)=\cO\pa{\sum_{i\in [n]}\frac{ \pa{ \sum_{j\in [\ell]} B_{ij}c_j }^2\log(m)\log(T)}{\Delta_{i,\min}}}.\]
 \end{theorem}
 
 The idea of the proof is quite simple once Assumption~\ref{reduce2exact} has been made. We can see that the approximation regret can be decomposed into $\ell$ sub-regrets, according to equation~\eqref{eq:ora2}. 
 Then, one must focus on the fact that the sub-regrets may dependent on each other and that the right gap (defined with the original reward function) must be obtained in the denominator of the final bound.

\subsection{Examples of \textsc{reduce2exact} problems}\label{subsec:ex}
Here, we present several problems belonging to \textsc{reduce2exact}. Each time, after a quick introduction of the problem, we translate it into our CMAB-T context, and finally show how it satisfies the \textsc{reduce2exact} criteria. The Travelling salesman problem (TSP) is treated in Appendix~\ref{app:tsp}.


\paragraph{Submodular maximization (e.g., probabilistic maximum coverage (PMC))} Here, we only expose the PMC example, noting that the same derivation can be applied to a monotone submodular\footnote{$f$ is monotone if for every $A\subset B$, we have $f(A)\leq f(B)$. It is submodular
if for every $A,B$ we have that
$f(A\cup B) + f(A\cap B) \leq f(A)+f(B)$.} function.
PMC is one of the main examples proposed by \citet{kong2021hardness}. Given a weighted bipartite graph $G=(L,R,E)$, with weights $\bmu^*\triangleq (\mu^*_{(u,v)})_{(u,v)\in E}$ (notice there are thus $n=\abs{E}$ arms, recalling that for us 
an arm is only something that produces an outcome, not something we can choose as an action, explaining why arms are not indexed by vertices here), the goal is to find an action $A\in \cA\triangleq \set{A\subset L:~\abs{A}=k}$, $k\in \N^*$, maximizing the expected number of influenced nodes in $R$, where each node
$v \in R$ can be independently influenced by $u\in A$ with probability $\mu^*_{(u,v)}$, i.e., maximizing
\(f(A,\bmu^*)\triangleq\sum_{v\in R}\pa{1-\prod_{u\in A:~(u,v)\in E} (1-\mu^*_{(u,v)})}.\)
This problem can be applied to the semi-bandit framework called \emph{the ad placement problem}, where $L$ are the web pages, $R$ are the users and $\mu^*_{(u,v)}$ is the probability that user $v$ clicks on the ad on web page $u$. In this application, the user's click probabilities are unknown and must be learned as the rounds progress. The Greedy oracle can provide an approximate solution with approximation ratio $\alpha = 1 - 1/e$ \citep{nemhauser1978analysis}. This setting fits \textsc{reduce2exact} as follows:
\begin{itemize}
    \item  $\ora_1(\bmu):$
    For $i\in [k]$, let
$\cE_i\triangleq\set{ (a_1,\dots,a_i): (a_1,\dots,a_{i-1})=E_{i-1}, a_i\in L\backslash E_{i-1} }$ and $r_i((a_1,\dots,a_i), \bmu)\triangleq f(\set{a_1,\dots,a_i},\bmu)$.
\item $\ora_2\pa{E_1,\dots,E_k}:$ Let $(a_1,\dots,a_k)\triangleq E_k$.
 $\ora_2$ returns $A=\set{a_1,\dots,a_k}$. 
\end{itemize}
 Let $A^i=\set{a_1,\dots,a_i}$ for some $i\in [k]$ and by abuse of notation, let $f = f(\cdot,\bmu^*)$.
 Informally, we see that in the above decomposition of the oracle, at each step, $\ora_1$ maximizes $f(A^i)$ with $A^{i-1}$ fixed, i.e., $\ora_1$ optimizes only on $a_i$ (we thus recover the greedy algorithm). It is precisely these sub-problems of finding $a_i$ that can be solved exactly. Then, we see that $\ora_2$ simply returns the last $A^i$ constructed. We will now prove the relation \eqref{eq:ora2}. 
The following is true using that $f$ is monotone submodular (this is actually the way \citet{nemhauser1978analysis} proved the approximation guarantee, and is true for any monotone submodular function):
\begin{align*}
    f(A^*)-f(A^i)
    &\leq \sum_{a\in A^*\bsl A^i} \pa{f(\set{a} \cup A^i)-f(A^{i+1})} + k\pa{f(A^{i+1})-f(A^i)}.
\end{align*}
Once the above relation is obtained, we can actually continue the original proof from \citet{nemhauser1978analysis}, skipping each step where we would need to use the property of $\ora_1$, thus leaving a term in the right-hand side for each time we skipped.
\begin{align*}
    &{f(A^*)-f(A^k)}={f(A^*)-f(A^{k-1})} - \pa{f(A^{k})-f(A^{k-1})}\\
    &\leq \pa{f(A^*)-f(A^{k-1})}\pa{1-\frac{1}{k}} + \frac{1}{k}\sum_{a\in A^*\bsl A^{k-1}} \pa{f(\set{a} \cup A^{k-1})-f(A^{k})}
    \\\dots&\leq f(A^*)\pa{1-\frac{1}{k}}^k + \sum_{i=1}^{k} \frac{\pa{1-\frac{1}{k}}^{i-1}}{k}\sum_{a\in A^*\bsl A^{k-i}} \pa{f(\set{a} \cup A^{k-i})-f(A^{k-i+1})}.
\end{align*}
Finally, since $\pa{1-\frac{1}{k}}^k\leq e^{-1}$, we get that $\Delta\pa{\ora_2\pa{E_1,\dots,E_k}} = (1-e^{-1})f(A^*)-f(A)$ is bounded by
\[\sum_{i=1}^{k} \frac{\pa{1-\frac{1}{k}}^{i-1}}{k}\abs{A^*\bsl A^{k-i}} \pa{r_{k-i+1}(E_{k-i+1}^*,\bmu^*)-r_{k-i+1}(E_{k-i+1},\bmu^*)}.\]
Note in passing that we recover the classical approximation if the right-hand-side was equal to$~0$. It is easy to see that the reward function $r_j$ satisfies Assumption~\ref{ass:smooth} with $\bB_j=\be_{[n]}$. We thus finally get our Assumption~\ref{reduce2exact}.
\paragraph{Online influence maximization (OIM)}
As the analysis mentioned above only uses submodularity, it can be extended to the problem of online influence maximization in a social network. A social network
is modeled as a directed graph $G = (V, E)$, with nodes $V$ representing users and edges $E$ representing connections.
For a node $i\in V$, a subset $A\subset V$, and a vector $\bx\in \set{0,1}^E$, let the predicate $A\reach{\bx} i$ hold if, in the graph defined by $G_\bx\triangleq\pa{V,\set{ij\in E, x_{ij}=1}}$, there is a forward path
from a node in $A$ to the node~$i$. If it holds, we say that $i$ is influenced by $A$ under~$\bx$. The goal is to find an action $A\in \cA\triangleq \set{A\subset V:~\abs{A}=k}$ maximizing the \emph{influence spread} $\sigma\pa{A,\bmu^*}\triangleq \EE{\abs{\set{i\in V,~A\reach{\bX} i}}},$ where $\bX\sim\otimes_{(u,v)\in E}\Bernoulli(\mu^*_{(u,v)})$. This model is called the \emph{independent cascade model} \citep{kempe2003maximizing,kempe2015maximizing}. A notable property to use the greedy oracle is that $\sigma$ is monotone submodular.
As the exact calculation of $\sigma$ is prohibitive, it is estimated by simulating the diffusion process, resulting in an approximation factor  $\alpha = {1-{e^{-1}}-\varepsilon}$ in the above greedy oracle analysis \citep{kempe2015maximizing,feige1998threshold,chen2010scalable}, with $\varepsilon>0$. In OIM, arms may be probabilistically triggered, and Assumption~\ref{ass:smooth} holds with the constants being all equal to $\max_{u\in V}\abs{\set{v\in V,~\set{u}\reach{\pa{\II{\mu^*_e>0}}_{e\in E}} v}}$, which is the largest number
of nodes any node can reach \citep{wang2017improving}. As previously, we thus get Assumption~\ref{reduce2exact} with $\bB_j$ being $\be_{[n]}$ times this constant.

\paragraph{Metric $k$-center}
This example and the following ones are less common for CMAB, but allow to well illustrate \textsc{reduce2exact}. Given a set of cities, one wants to build $k$ warehouses in different cities and minimize the maximum distance of a city to a warehouse. Formally, given a complete undirected weighted graph $G=(V,E)$ whose distances $d(v_i,v_j)$ satisfy the triangle inequality, the goal is to find an action $A\in \cA\triangleq\set{A\subset V:~\abs{A}=k}$ that minimizes $\max_{v\in V}d(v, A)$. We can consider the semi-bandit setting where
the set of base arms is $E$, $\bmu^*\triangleq \pa{d(v_i,v_j)}_{(v_i,v_j)\in E}$ and the feedback set $S$ includes the edges of the graph induced by the chosen action. We can target an approximation regret with $\alpha=1/2$ using the following oracle (which is simply the standard greedy algorithm for this problem).
\begin{itemize}
    \item $\ora_1(\bmu):$
    For $i\in [k]$, let
$\cE_i\triangleq\set{ (a_1,\dots,a_i): (a_1,\dots,a_{i-1})=E_{i-1}, a_i\in V\backslash E_{i-1} }$ and $r_i((a_1,\dots,a_i), \bmu)\triangleq \min_{j\in [i-1]} \mu_{a_i,a_j}$ (with $r_1=0$). One can notice we thus have 
$E_i^*\in \argmax_{(a_1,\dots,a_i)\in \cE_i} d(a_i,\set{a_1,\dots,a_{i-1}})=\argmax_{(a_1,\dots,a_i)\in \cE_i} d(a_i,E_{i-1}).$
\item $\ora_2\pa{E_1,\dots,E_k}:$ Let $(a_1,\dots,a_k)\triangleq E_k$.
 $\ora_2$ returns $A=\set{a_1,\dots,a_k}$.
Let $w\in \argmax_{v\in V} d(v,A)$.
\end{itemize}
We thus have:
\begin{align*}
    \frac{1}{2}\max_{v\in V}d(v,A) - \max_{v\in V}d(v,A^*) &\leq  \frac{1}{2}\pa{\max_{v\in V}d(v,A) - \min_{a\in A\cup\set{w}}\min_{a'\in A\backslash\set{a}}d(a,a')}\\
    &= \frac{1}{2}\max_{j\in [k-1]} \pa{\max_{v\in V}d(v,A) -  d(a_{j+1},E_j)}\\
    &\leq  \frac{1}{2}\max_{j\in [k-1]} \pa{\max_{v\in V}d(v,E_j) -  d(a_{j+1},E_j)}\\
    &= \frac{1}{2}\max_{j\in [k-1]} \pa{r_{j+1}\pa{E_{j+1}^*,\bmu^*} -  r_{j+1}\pa{E_{j+1},\bmu^*}}.
\end{align*}
Where the first inequality is deduced as follows: the map $f:~v\mapsto \argmin_{a^*\in A^*} d(a^*,v)$ defines a partition of $V$ into $k=\abs{A^*}$ clusters. By the the pigeonhole principle, one cluster contains 2 different points $a,a'\in A\cup\set{w}$ (simply because its size is $k+1$). We can assume $a'\in A$ without loss of generality. Thus,
since $f(a)=f(a')$, we get
 $d(a,a')\leq d(a,f(a))+d(a',f(a'))=d(a,A^*)+d(a',A^*)\leq 2\max_{v\in V}d(v,A^*)$. 

\paragraph{Vertex cover}
The problem consists, given an undirected graph $G=(V,E)$, in finding a set of vertices with minimal cost to cover all the edges of $E$. Formally, with $\bmu^*\in \R_+^{V}$, the goal is to find an action $A\in \cA\triangleq \set{A\subset V: \forall (u,v)\in E, u\in A\text{ or }v\in A}$ that minimizes $\be_{A}\transpose \bmu^* $. The semi-bandit feedback is defined directly as $S=A$. We can target an approximation regret with $\alpha=1/2$ using the following linear programming (LP) relaxation oracle.
\begin{itemize}
    \item $\ora_1(\bmu):$
    Let\footnote{The LP relaxation of vertex cover is half-integral, so that we can allow each variable to be in $\set{0,1/2,1}$ rather than the interval from $0$ to $1$.}
$\cE_1\triangleq\set{ \bx\in \set{0,1/2,1}^V: \forall (u,v)\in E,~ x_u+x_v\geq 1}$ and $r_1(\bx, \bmu)\triangleq -\bx\transpose\bmu.$
\item $\ora_2\pa{E_1}:$ Let $\bx\triangleq E_1$.
 $\ora_2$ returns $A=\set{v\in V:x_v\geq 1/2}\in \cA$ (since $\forall~(u,v)\in E,x_u+x_v\geq 1$, so $x_u\geq 1/2$ or $x_v\geq 1/2$, so $u\in A$ or $v\in A$). Let $\bx^*\triangleq E_1^*$.
\end{itemize}
We have:
\begin{align*}
    \frac{1}{2}\be_A\transpose\bmu^* - \be_{A^*}\transpose\bmu^* \leq  \bx\transpose\bmu^* - \be_{A^*}\transpose\bmu^*
    \leq \bx\transpose\bmu^* - {\bx^*}\transpose\bmu^*
    = r_1(E_1^*,\bmu^*) - r_1(E_1,\bmu^*).
\end{align*}
Notice, $r_1$ satisfies
Assumption~\ref{ass:smooth}
with the constants being $1$, using $\bx\in \set{0,1/2,1}^V,$ so that $\bx\leq \be_A$.

\paragraph{Max-Cut} Given an undirected weighted graph $G=(V,E)$, with weights $\bmu^*\triangleq (\mu^*_{(u,v)})_{(u,v)\in E}$, the goal is to find an action $A\in \cA\triangleq \set{A\subset V}$ maximizing the total weight of the edges between $A$ and its complement, i.e., $\frac{1}{2}\sum_{(u,v)\in E} \mu^*_{(u,v)}\pa{1-y_uy_v} $, where $\by\triangleq\be_A-\be_{V\backslash A}$. We consider the semi-bandit context where $E$ is the set of arms, and the feedback set includes the edges between $A$ and its complement, i.e., $S=\set{(u,v)\in E,~y_uy_v=-1}$. To include randomization within the oracle, we can extend the action space $\cA$ to the set of probability measures on $\set{A\subset V}$, replacing $\Delta(A)$ by its expectation on $A$, as a function of the distribution of $A$. 
The polynomial-time approximation algorithm for Max-Cut with the best known approximation ratio \citep{goemans1995improved} uses semidefinite programming and randomized rounding, and achieves an approximation ratio
of
$\alpha ={\frac {2}{\pi }}\min _{{0\leq \theta \leq \pi }}{\frac {\theta }{1-\cos \theta }} \approx 0.878$. In our context, it can be defined as follows.
\begin{itemize}
    \item $\ora_1(\bmu):$
    We define
$\cE_1\triangleq\set{ (\bv_u)_{u\in V}\in \set{\bv\in \R^V: \norm{\bv}_2=1}^V}$ and $r_1((\bv_u)_{u\in V}, \bmu)\triangleq \frac{1}{2}\sum_{(u,v)\in E} \mu_{(u,v)}\pa{1-\bv_u\transpose\bv_v} $.
\item $\ora_2\pa{E_1}:$ Let $(\bv_u)_{u\in V}\triangleq E_1$.  $\ora_2$ returns the distribution of $A=\set{u\in V,~\bv_u\transpose\bZ\geq 0}$, where $\bZ\sim\cU\pa{\set{\bv\in \R^V: \norm{\bv}_2=1}}$.
\end{itemize}
The following is proved by \citet{goemans1995improved}:
\begin{align*}
\sum_{(u,v)\in E} \mu_{(u,v)}\PP{(u,v)\in S}&=\EE{\sum_{(u,v)\in E} \mu_{(u,v)}\frac{1-\sign\pa{\bv_u\transpose \bZ}\sign\pa{\bv_v\transpose \bZ}}{2}}
\\&= {\sum_{(u,v)\in E} \mu_{(u,v)}\frac{\arccos\pa{\bv_u\transpose\bv_v}}{\pi}}\\
&\geq  \alpha{\sum_{(u,v)\in E} \mu_{(u,v)}\frac{1-{\bv_u\transpose \bv_v}}{2}}=\alpha r_1((\bv_u)_{u\in V}, \bmu).\end{align*}
Thus, $r_1$ satisfies Assumption~\ref{ass:smooth} with the constants being $1/\alpha$. We also get, with $\by^*\triangleq\be_{A^*}-\be_{V\backslash A^*}$: 
\begin{align*}
\Delta\pa{\ora_2\pa{E_1}}&={\alpha}\sum_{(u,v)\in E} \mu^*_{(u,v)}\frac{1-y_u^*y_v^*}{2} - \EE{\sum_{(u,v)\in E} \mu^*_{(u,v)}\frac{1-\sign\pa{\bv_u\transpose \bZ}\sign\pa{\bv_v\transpose \bZ}}{2}}\\
&\leq\alpha\pa{r_1(E_1^*,\bmu^*)- r_1(E_1,\bmu^*)}.
\end{align*}



\section{Conclusion}\label{sec:ccl}
In this article, our main objective is to further expand the "approximation regret scope" of the \textsc{cts} policy. We not only expand it to probabilistically triggered arms (which was one of the open questions by \citet{kong2021hardness}), but we also consider a broader class of oracles compatible with \textsc{cts}. More precisely, we propose a condition, \textsc{reduce2exact}, which may seem unnatural at first, but which in fact simply expresses that sub-problems that can be solved exactly must be hidden in the original approximation problem, and that the approximation oracle exploit them to output the final solution. Knowing that the majority of approximation algorithms use one or more relaxations to an exact problem (e.g., solving a convex programming relaxation to obtain a fractional solution and then rounding this fractional solution to get a feasible solution), our assumption falls within the range of many 
CMAB-T settings. From this reduction, we naturally obtain the standard tight regret bound $\cO\pa{\log(T)/\Delta_{\min}}$. This is the first tight bound for the approximation regret on non-exact oracles.

As future work, it may be interesting to explore other CMAB-T problems where the \textsc{reduce2exact} condition does (or doesn't) hold. We also think our setting should be generalizable to the \emph{budgeted regret} setting without much difficulty (see \citet{pmlr-v89-perrault19a, perraultbudgeted2020, perrault20rsds} for examples with an approximation oracle).
Finally, we have that the approximation regret is in some way conservative compared to the greedy regret. For a given oracle, we can easily consider the equivalent of the greedy regret for that oracle. An interesting investigation would then be to extend the work of \citet{kong2021hardness} in this direction, considering other types of oracle.

\clearpage
\setlength{\bibsep}{5.1pt}
\bibliography{example}

\section*{Checklist}

\begin{enumerate}

\item For all authors...
\begin{enumerate}
  \item Do the main claims made in the abstract and introduction accurately reflect the paper's contributions and scope?
    \answerYes{See Abstract and section~\ref{sec:intro}.}
  \item Did you describe the limitations of your work?
    \answerYes{See the paragraph "Other known limitations of CTS" and the second paragraph in section~\ref{sec:ccl}.}
  \item Did you discuss any potential negative societal impacts of your work?
    \answerNA{This work does not have any potential
negative societal impacts.}
  \item Have you read the ethics review guidelines and ensured that your paper conforms to them?
    \answerYes{}
\end{enumerate}

\item If you are including theoretical results...
\begin{enumerate}
  \item Did you state the full set of assumptions of all theoretical results?
    \answerYes{See section~\ref{sec:model}.}
        \item Did you include complete proofs of all theoretical results?
    \answerYes{See all sections
in Appendix.}
\end{enumerate}

\item If you ran experiments...
\begin{enumerate}
  \item Did you include the code, data, and instructions needed to reproduce the main experimental results (either in the supplemental material or as a URL)?
    \answerNA{}
  \item Did you specify all the training details (e.g., data splits, hyperparameters, how they were chosen)?
    \answerNA{}
        \item Did you report error bars (e.g., with respect to the random seed after running experiments multiple times)?
    \answerNA{}
        \item Did you include the total amount of compute and the type of resources used (e.g., type of GPUs, internal cluster, or cloud provider)?
    \answerNA{}
\end{enumerate}

\item If you are using existing assets (e.g., code, data, models) or curating/releasing new assets...
\begin{enumerate}
  \item If your work uses existing assets, did you cite the creators?
    \answerNA{}
  \item Did you mention the license of the assets?
    \answerNA{}
  \item Did you include any new assets either in the supplemental material or as a URL?
    \answerNA{}
  \item Did you discuss whether and how consent was obtained from people whose data you're using/curating?
    \answerNA{}
  \item Did you discuss whether the data you are using/curating contains personally identifiable information or offensive content?
    \answerNA{}
\end{enumerate}

\item If you used crowdsourcing or conducted research with human subjects...
\begin{enumerate}
  \item Did you include the full text of instructions given to participants and screenshots, if applicable?
    \answerNA{}
  \item Did you describe any potential participant risks, with links to Institutional Review Board (IRB) approvals, if applicable?
    \answerNA{}
  \item Did you include the estimated hourly wage paid to participants and the total amount spent on participant compensation?
    \answerNA{}
\end{enumerate}

\end{enumerate}

\clearpage
\appendix

\section{Proof of Theorem~\ref{thm:tsexact}}\label{app:tsexact}
In the following proof, we will treat both algorithms at the same time, detailing the steps where there is a difference. Thus, for Algorithm~\ref{algo:tsbeta}, we use the convention $\beta=1$ in the following. 
We let $B\triangleq\norm{\bB}_\infty$ and $0<\varepsilon<\Delta_{\min}/(2B({m^*}^2+1))$. For the two algorithms, we consider the following events for
any time step $t\in \N^*$:
\newline
    ${\bullet}~\fZ_t\triangleq\set{\Delta_t>0}$,
    \newline
    ${\bullet}~\fB_t\triangleq \set{\norm{\bp\pa{A_t}\odot\bB\odot\pa{\bar\bmu_{t-1} - \bmu^*} }_1>{{\Delta_{\min}}/2-B({m^*}^2+1)\eps} }$,
    \newline
    $ {\bullet}~\fC_t\triangleq\set{{\norm{\bp\pa{A_t}\odot\bB\odot\pa{\btheta_{t}-\bmu^*} }_1}> \Delta_t-B\pa{{m^*}^2+1}\eps}$,
    \newline 
    ${\bullet}~\fD_t\triangleq \set{\norm{\bp\pa{A_t}\odot\bB\odot\pa{\btheta_{t}-\bar\bmu_{t-1}} }_1\geq\sqrt{\sum_{i\in [n]}\frac{2{\log\pa{2^n\pa{1+\ceil{\log_2\pa{{1}/{p^*}}}}^nT}}p_i(A_t)^2B_i^2\beta}{N_{i,t-1}}}}$.
\newline
We decompose the regret analysis into several steps, each step corresponding to a filtration of the regret against a combination of these events.
\paragraph{Step 1: bound under $\fZ_t\wedge\neg\fC_t$}The filtered regret bound 
\[\EE{\sum_{t=1}^T\Delta_t\II{\fZ_t\wedge\neg\fC_t}}\leq \Delta_{\max}\frac{cm^*}{p^*\varepsilon^2}\pa{\frac{c'}{\varepsilon ^4}}^{m^*}\] is deduced from the following two lemmas, considering the following events for a subset $Z\subset [n]$: 
\begin{align*}&\fR(\btheta',Z)\triangleq\\ &\set{Z\subset \textsc{t}\pa{A}~\text{s.t.}~A=\ora\pa{\btheta'},~\norm{\bp\pa{A}\odot\bB\odot\pa{\btheta'-\bmu^*}}_1>\Delta\pa{A}-B({m^*}^2+1)\eps},\end{align*}
\[\fS_t\pa{Z}\triangleq \set{\forall \btheta' \text{ s.t. } \norm{\pa{\bmu^*-\btheta'}\odot\be_{Z}}_\infty\leq \eps,~ \fR(\btheta'\odot\be_{Z}+\btheta_t\odot\be_{Z^c},Z) \text{ holds} },\]
\[\fT_t\pa{Z}\triangleq \set{\norm{\pa{\bmu^*-\btheta_t}\odot\be_{Z}}_\infty> \eps}.\]
\begin{lemma}
 $$\fZ_t,\neg \fC_t\imp \exists Z\subset \textsc{t}\pa{A^*},~Z\neq \emptyset~\text{s.t. the event }\fS_t\pa{Z}\wedge\fT_t\pa{Z}\text{ holds.}$$
 \label{lem:fSfT+}
\end{lemma}

 \begin{lemma}\label{lem:fSfT++} There are two constants $c,c'$ such that 
\[ \sum_{Z\subset \textsc{t}\pa{A^*}, Z\neq \emptyset}\EE{\sum_{t=1}^{T} \II{\fS_t\pa{Z},\fT_t\pa{Z}}}\leq \frac{cm^*}{p^*\varepsilon^2}\pa{\frac{c'}{\varepsilon ^4}}^{m^*}.\]
\end{lemma}

\begin{proof}[Proof of Lemma~\ref{lem:fSfT+}]
It is sufficient to prove that
 \begin{align}\fZ_t,\neg \fC_t\imp \exists Z\subset \textsc{t}\pa{A^*},~Z\neq \emptyset~\text{s.t. }\fS_t\pa{Z}\text{ holds,}\label{rel:toprovets}\end{align}
 because $\neg \fC_t$ and $\fS_t\pa{Z}$ together imply  $\fT_t\pa{Z}$. Then, we get \eqref{rel:toprovets} in a similar way as Lemma~2 from \citet{huyuk2019analysis}, which is possible as their Assumption 3 is implied by our Assumption~\ref{ass:smooth}. More precisely, the only places where we use our Assumption~\ref{ass:smooth} instead of their Assumption 3 are in cases $1a,2a,...$, when we show that $\fR(\btheta'\odot\be_{Z}+\btheta_t\odot\be_{Z^c},Z)$ holds. The detailed proof is given in the following.

 We first consider the choice $Z=Z_1=\textsc{t}\pa{A^*}$. 
Two cases can be distinguished:
\begin{itemize}
    \item[1a)] $\forall \btheta'$ s.t. $\norm{ \pa{\bmu^*-\btheta'}\odot\be_{Z_1}}_{\infty}\leq \eps $, we have $Z_1\subset \textsc{t}\pa{\ora\pa{\btheta'\odot\be_{Z_1}+\btheta_t\odot\be_{{Z_1}^c}}}$.
        \item[1b)] $\exists \btheta'$ s.t. $\norm{ \pa{\bmu^*-\btheta'}\odot\be_{Z_1}}_{\infty}\leq \eps $ such that $Z_1\not\subset \textsc{t}\pa{\ora\pa{\btheta'\odot\be_{Z_1}+\btheta_t\odot\be_{{Z_1}^c}}}$.
\end{itemize}
\textbf{1a)} For the first case, consider any vector $\btheta'$ such that $\norm{ \pa{\bmu^*-\btheta'}\odot\be_{Z_1}}_{\infty}\numrel{\leq}{finfirstlemm1} \eps$ and let $A\numrel{=}{rel0lem1}\ora\pa{\btheta'\odot\be_{Z_1}+\btheta_t\odot\be_{{Z_1}^c}}$. We can write
$$r\pa{A,\btheta'\odot\be_{Z_1}+\btheta_t\odot\be_{{Z_1}^c}} \numrel{\geq}{rel1lem1} r\pa{A^*,\btheta'\odot\be_{Z_1}+\btheta_t\odot\be_{{Z_1}^c}} \numrel{\geq}{rel2lem1} r\pa{A^*,\bmu^*} - B m^*\eps,$$
where \eqref{rel1lem1} is from \eqref{rel0lem1}, and \eqref{rel2lem1} is from \eqref{finfirstlemm1}.
This rewrites as
$$r\pa{A,\btheta'\odot\be_{Z_1}+\btheta_t\odot\be_{{Z_1}^c}} \geq r\pa{A^*,\bmu^*} - Bm^*\eps >  r\pa{A^*,\bmu^*} - B\pa{{m^*}^2+ 1}\eps,$$
so $\fR_t(\btheta'\odot\be_{Z_1}+\btheta_t\odot\be_{{Z_1}^c},Z_1)$ holds. Therefore, we have proved that $\fS_t\pa{Z_1}$ holds.

\textbf{1b)} For the second case, we have some vector $\btheta'$ such that $\norm{ \pa{\bmu^*-\btheta'}\odot\be_{Z_1}}_{\infty}\numrel{\leq}{lastlemma1} \eps  $, and some action $A=\ora\pa{\btheta'\odot\be_{Z_1}+\btheta_t\odot\be_{{Z_1}^c}}$  such that $Z_1\not\subset\textsc{t}\pa{ A}$.
We consider $Z_2=Z_1\cap \textsc{t}\pa{ A}\neq Z_1$. We first prove that $Z_2\neq \emptyset$ by showing that if an action $A'$ is such that $Z_1\cap\textsc{t}\pa{A'}\numrel{=}{relS'lem1}\emptyset$, then $A'\neq A$ as it has a lower reward value than that of $A^*$:
\begin{align*}
    r\pa{A',\btheta'\odot\be_{Z_1}+\btheta_t\odot\be_{{Z_1}^c}} \numrel{=}{relS'lem1bis} r\pa{A',\btheta_t} 
    &\numrel{\leq}{relStlem1}
    r\pa{A_t,\btheta_t} 
    \\&\numrel{\leq}{relCtlem1}
    r\pa{A^*, \bmu^*}-B\pa{{m^*}^2+1}\eps
    \\&< r\pa{A^*,\bmu^*}-Bm^*\eps
    \\&\numrel{\leq}{relepsslem1} r\pa{A^*,\btheta'\odot\be_{Z_1}+\btheta_t\odot\be_{{Z_1}^c}},
\end{align*}
where \eqref{relS'lem1bis} is from \eqref{relS'lem1}, \eqref{relStlem1} is from the definition of $A_t$, \eqref{relCtlem1} is from $\neg \fC_t$ and \eqref{relepsslem1} is from \eqref{lastlemma1}.
Now, we again distinguish two cases:
\begin{itemize}
    \item[2a)] $\forall \btheta''$ s.t. $\norm{ \pa{\bmu^*-\btheta''}\odot\be_{Z_2}}_{\infty}\leq \eps $, we have $Z_2\subset \textsc{t}\pa{\ora\pa{\btheta''\odot\be_{Z_2}+\btheta_t\odot\be_{{Z_2}^c}}}$.
        \item[2b)] $\exists \btheta''$ s.t. $\norm{ \pa{\bmu^*-\btheta''}\odot\be_{Z_2}}_{\infty}\leq \eps $ such that $Z_2\not\subset \textsc{t}\pa{\ora\pa{\btheta'\odot\be_{Z_2}+\btheta_t\odot\be_{{Z_2}^c}}}$.
\end{itemize}
Notice that when $\norm{ \pa{\bmu^*-\btheta''}\odot\be_{Z_2}}_{\infty}\numrel{\leq}{rellem1eps}  \eps$, then
\begin{align}
    r\pa{A,\btheta''\odot\be_{{Z_2}}+\btheta_t\odot\be_{{{Z_2}}^c}}&\geq r\pa{A,\btheta'\odot\be_{{Z_1}}+\btheta_t\odot\be_{{{Z_1}}^c}} - 2B(m^*-1)\eps. \label{rel2epslemm1}
\end{align}
Indeed, \eqref{rel2epslemm1} is a consequence of 
\begin{align*}
&\norm{\pa{\btheta'\odot\be_{{Z_1}}+\btheta_t\odot\be_{{{Z_1}}^c} - \btheta''\odot\be_{{Z_2}}-\btheta_t\odot\be_{{{Z_2}}^c}}\odot\be_{\textsc{t}(A)}}_1 \\&=
\norm{\pa{\btheta' - \btheta''}\odot\be_{Z_2}}_1
\\&\leq \norm{\pa{\bmu^* - \btheta'}\odot\be_{Z_2}}_1 + \norm{\pa{\bmu^* - \btheta''}\odot\be_{Z_2}}_1
\\&\leq 2(m^*-1)\eps,
\end{align*}
where we used \eqref{rellem1eps}, \eqref{lastlemma1} and that $Z_2$ is strictly included in $Z_1$.

\textbf{2a)} For the first case, considering any vector $\btheta''$ such that $\norm{ \pa{\bmu^*-\btheta''}\odot\be_{Z_2}}_{\infty}\leq \eps $, we have with $\tilde A=\ora\pa{\btheta''\odot\be_{Z_2}+\btheta_t\odot\be_{{Z_2}^c}}$ that
\begin{align*}r\pa{\tilde A, \btheta''\odot\be_{{Z_2}}+\btheta_t\odot\be_{{{Z_2}}^c}}&\geq r\pa{A, \btheta''\odot\be_{{Z_2}}+\btheta_t\odot\be_{{{Z_2}}^c}}\\&\numrel{\geq}{releps21lem1}
r\pa{A,\btheta'\odot\be_{{Z_1}}+\btheta_t\odot\be_{{Z_1}^c}} - 2B(m^*-1)\eps
\\&\geq r\pa{A^*,\btheta'\odot\be_{{Z_1}}+\btheta_t\odot\be_{{Z_1}^c}} - 2B(m^*-1)\eps\\&\numrel{\geq}{otherlem1} 
r\pa{A^*,\bmu^*} -Bm^*\eps -2B(m^*-1)\eps 
\\&\geq r\pa{A^*,\bmu^*} -B({m^*}^2+1)\eps
,
\end{align*}
where \eqref{releps21lem1} uses \eqref{rel2epslemm1} and \eqref{otherlem1} uses \eqref{lastlemma1}.
Therefore, $\fR_t(\btheta'\odot\be_{{Z_2}}+\btheta_t\odot\be_{{{Z_2}}^c},{Z_2})$ holds, and we proved that $\fS_t(Z_2)$ holds. 

\textbf{2b)} For the second case, we have a vector $\btheta''$ such that $\norm{ \pa{\bmu^*-\btheta''}\odot\be_{Z_2}}_{\infty}\leq \eps $ and an action $\tilde A=\ora\pa{\btheta''\odot\be_{Z_2}+\btheta_t\odot\be_{{Z_2}^c}}$ such that $Z_2\not\subset \textsc{t}\pa{\tilde A}$. We consider ${Z_3}=Z_2\cap \textsc{t}\pa{\tilde A}$. Again, $Z_3\neq~\emptyset$ because for any $A''$ such that $\textsc{t}\pa{A''}\cap Z_2 = \emptyset$, we have $A''\neq \ora\pa{\btheta''\odot\be_{Z_2}+\btheta_t\odot\be_{{Z_2}^c}}$:
\begin{align*}
    r\pa{A'',\btheta''\odot\be_{Z_2}+\btheta_t\odot\be_{{Z_2}^c}} = r\pa{A'',\btheta_t} 
    &\leq
    r\pa{A_t,\btheta_t} 
    \\&\leq
    r\pa{A^*, \bmu^*}-B\pa{{m^*}^2+ 1}\eps
    \\&< r\pa{A^*, \bmu^*}-Bm^*\eps -2B(m^*-1)\eps
    \\&\leq r\pa{A,\btheta''\odot\be_{Z_2}+\btheta_t\odot\be_{{Z_2}^c}},
\end{align*}
where the last inequality is obtained in the same way as in inequalities from \eqref{releps21lem1} to \eqref{otherlem1}.

We could repeat the above argument and each time the size $Z_i$ is decreased by at least $1$. Thus, after at most $m^*-1$ steps, since $m^*+ 2(m^*-1) + 2(m^*-2)+\dots+ 2 = {m^*}^2<{m^*}^2 +1$, we could reach the end and find a $Z_i\neq \emptyset$ such that $\fS_t\pa{Z_i}$ holds.  
 \end{proof}

\begin{proof}[Proof of Lemma~\ref{lem:fSfT++}]
Let $Z=\set{z_0,\dots,z_{\abs{Z}-1}}\subset \textsc{t}\pa{A^*}, Z\neq \emptyset$.
If $\abs{Z}=1$, we let
$$\eta_{q,0}\triangleq \set{t\geq 1,~\abs{\set{t'\in [t-1],~\fS_{t'}\pa{Z}\wedge\neg\fT_{t'}\pa{Z}\wedge\set{z_{0}\in S_{t'}}}}=q}.$$
Else, for $t>1$, we recursively define
$$c_{t+1}\triangleq c_t+\II{\fS_t\pa{Z}\wedge\neg\fT_t\pa{Z}\wedge\set{z_{c_t}\in S_t}}\mod{\abs{Z}},$$
with $c_1\triangleq 0$ and let 
$$\eta_{q,k}\triangleq \set{t\geq 1,~c_t=k,~\abs{t'\in [t-1],~c_{t'}=k\neq c_{t'+1}}=q}.$$
Notice that for $\tau\geq \inf \eta_{q,0}$, we have $N_{i,\tau-1}\geq q$ for all $i\in Z$.
We have\begin{align*}
    &\EE{\sum_{t=1}^{T} \II{\fS_t\pa{Z},\fT_t\pa{Z}}} = \sum_{q\geq 0}\sum_{k= 0}^{\abs{Z}-1} \EE{\sum_{t\in\eta_{q,k}}\II{\fS_t\pa{Z},\fT_t\pa{Z}}}
    \\&\leq 
    \frac{1}{p^*}\sum_{q\geq 0}\sum_{k= 0}^{\abs{Z}-1} \EE{\sum_{t\in\eta_{q,k}}\II{\fS_t\pa{Z},\fT_t\pa{Z}, z_{c_t}\in S_t}}
    \\&\leq\frac{1}{p^*}\sum_{q\geq 0}\sum_{k= 0}^{\abs{Z}-1} \pa{\EE{\sup_{\tau\geq{\inf\eta_{q,k}}}\frac{1}{\PPc{\neg\fT_\tau\pa{Z}}{\cH_\tau}}}-1}
    \\&\leq \frac{\abs{Z}}{p^*}\sum_{q\geq 0}\pa{\EE{\sup_{\tau\geq{\inf\eta_{q,0}}}\prod_{i\in Z}\frac{1}{\PPc{\abs{\theta_{i,\tau}-\mu_i^*}\leq \eps}{\cH_\tau}}}-1}.
\end{align*}
From the initialization phase, we can assume that the event 
\[\fM_t\triangleq \set{\forall i\in [n],~N_{i,t-1}\geq 1}\]
holds (under the complementary event, we have the upper bound $n$). If there is no initialization, we can have $q=0$ in the following, noticing that when $\theta_{i,t}$ is uniform on $[a,b]$, then the probability $\PPc{\abs{\theta_{i,t}-\mu_i^*}\leq \eps}{\cH_t}$ is equal to $2\eps/(b-a)$. We are thus interested in bounding 
$$\frac{\abs{Z}}{p^*}\underbrace{\sum_{q\geq 1} \pa{\EE{\sup_{\tau\geq{\tau_q}}\prod_{i\in Z}\frac{1}{\PPc{\abs{\theta_{i,\tau}-\mu_i^*}\leq \eps}{\cH_\tau}}}-1}}_{\numterm{rel:tst2}},$$
with $\tau_q\triangleq \inf\eta_{q,0}$. We have
\begin{align*}\eqref{rel:tst2}&\leq
    \sum_{q\geq 1}\EE{\sup_{\tau\geq{\tau_q}}\sum_{Z'\subset Z,~Z'\neq \emptyset}{\prod_{i\in Z'}\pa{\frac{1}{\PPc{\abs{\theta_{i,\tau}-\mu_i^*}\leq \eps}{\cH_\tau}}-1}}}\\&\leq
   \sum_{q\geq 1} \sum_{Z'\subset Z,~Z'\neq \emptyset} \underbrace{\EE{\sup_{\tau\geq{\tau_q}}{\prod_{i\in Z'}\pa{\frac{1}{\PPc{\abs{\theta_{i,\tau}-\mu_i^*}\leq \eps}{\cH_\tau}}-1}}}}_{\numterm{rel:ts2}}\end{align*}
   Then, we can take a union bound on the counters:
   \begin{align*}
    \eqref{rel:ts2}&\leq
   \sum_{\bk\in [q..\infty)^{Z'}}
   \EE{\sup_{\tau\geq{\tau_q}}\II{\forall i \in Z',~N_{i,\tau-1}=k_i}{\prod_{i\in Z'}\pa{\frac{1}{\PPc{\abs{\theta_{i,\tau}-\mu_i^*}\leq \eps}{\cH_\tau}}-1}}}.
\end{align*}

From this point, there are two distinct analysis depending on whether we consider Algorithm~\ref{algo:tsbeta} or Algorithm~\ref{algo:tsgauss}.

\textbf{For Algorithm~\ref{algo:tsbeta}:}
\newline
For any arm $i\in [n]$, $k_i\in \N$, we define
$p_{i,k_i}$ as the probability of $\abs{\tilde\theta_{i,k_i}-\mu_i^*}\leq \eps$, where $\tilde\theta_{i,k_i}$ is a sample from the posterior of arm $i$ when there are $k_i$ observations of arm $i$ (i.e., $p_{i,k_i}$ is a random variable measurable with respect to those $k_i$ independent draws of arm $i$).
We have 
\begin{align*}
\EE{\sup_{\tau\geq{\tau_q}}\II{\forall i \in Z',~N_{i,\tau-1}=k_i}{\prod_{i\in Z'}\pa{\frac{1}{\PPc{\abs{\theta_{i,\tau}-\mu_i^*}\leq \eps}{\cH_\tau}}-1}}}&=
\EE{\prod_{i\in Z'}\pa{\frac{1}{p_{i,k_i}}-1}},
     \\&=
\prod_{i\in Z'}\EE{\pa{\frac{1}{p_{i,k_i}}-1}}.
     \end{align*}
From Lemma~5,6 in \citet{Wang2018}, we know that  
\[{\EE{\frac{1}{p_{i,k_i}}}}\leq \left\{
    \begin{array}{ll}
        4/\eps^2 & \mbox{for every } k_i\geq 0 \\
        1+6c''\cdot{e^{-\eps^2 k_i/2}}\eps^{-2} +\frac{2}{e^{\eps^2k_i/8}-2}& \mbox{if }  k_i> 8/\eps^2,
    \end{array}
\right.\]
for some universal constant $c''$.
There are thus two cases:
     If $q>8/\eps^2$, then some simple calculations show that $\sum_{Z'\subset Z,~Z'\neq \emptyset}\eqref{rel:ts2}$ is bounded by a term of the form
     $e^{-\eps^2 q/8}\pa{c'\eps^{-4}}^{\abs{Z}} ,
$ where $c'$ is a universal constant,
and if $q\leq 8/\eps^2$, then $\sum_{Z'\subset Z,~Z'\neq \emptyset}\eqref{rel:ts2}$ is bounded by 
$
     \pa{  c\eps^{-4}}^{\abs{Z}} ,
$
where $c$ is a universal constant.
Summing over $q\geq 1$, we thus get the desired result.

\textbf{For Algorithm~\ref{algo:tsgauss}:}
\newline
One can notice that for all $i\in Z'$, all $k_i\geq q$, $\II{N_{i,\tau-1}=k_i}\pa{\frac{1}{\PPc{\abs{\theta_{i,\tau}-\mu_i^*}\leq \eps}{\cH_\tau}}-1}$ is of the form  $\II{N_{i,\tau-1}=k_i}g_i\pa{\abs{\bar\mu_{i,\tau-1}-\mu_i^*}}$, with $g_i$ being an increasing function on $\R_+$. Indeed, we see that the conditional distribution of $\theta_{i,\tau}-\bar\mu_{i,\tau-1}$ is $\cN\pa{0,\beta N_{i,\tau-1}^{-1}/4}$, which is symmetric, so we have 
\[\PPc{\abs{\theta_{i,\tau}-\mu_i^*}\leq \eps}{\cH_\tau}=\PPc{\abs{\theta_{i,\tau}-\bar\mu_{i,\tau-1}+\abs{\bar\mu_{i,\tau-1}-\mu_i^*}}\leq \eps}{\cH_\tau}.\]
In addition, under $\II{N_{i,\tau-1}=k_i}$, the conditional distribution of $\theta_{i,\tau}-\bar\mu_{i,\tau-1}$ does not depend on the history, but only on $k_i$. Therefore, the above probability is a function of $\abs{\bar\mu_{i,\tau-1}-\mu_i^*}$ and so the function $g_i$ exists. It
is increasing on $\R_+$ because for any fixed $\sigma>0$,
\[\frac{\partial}{\partial x}\int_{x-\eps}^{x+\eps}\frac{1}{\sqrt{2\pi\sigma^2}}e^{-\frac{u^2}{2\sigma^2}}\mathrm{d}u=
\frac{1}{\sqrt{2\pi\sigma^2}}\pa{e^{-\frac{(x+\eps)^2}{2\sigma^2}}-e^{-\frac{(x-\eps)^2}{2\sigma^2}}}< 0\text{ for }x> 0.
\]
In particular, we can consider the inverse function $g_i^{-1}$.
We now want to use a stochastic dominance argument in order to treat the outcomes as if they were Gaussian: we have for any $\bk\in [q..\infty)^{Z'}$,
\begin{align}\nonumber
    &\EE{\sup_{\tau\geq{\tau_q}}\prod_{i\in Z'}\pa{\II{N_{i,\tau-1}=k_i}g_i\pa{\abs{\bar\mu_{i,\tau-1}-\mu_i^*}}}}\\&=
    \EE{\sup_{\tau\geq{\tau_q}}\prod_{i\in Z'}\pa{\II{N_{i,\tau-1}=k_i}\int_{0}^\infty \II{g_i\pa{\abs{\bar\mu_{i,\tau-1}-\mu_i^*}}\geq u_i}\mathrm{d}u_i}}\nonumber\\&\leq\int_{\bu\in \R_+^{Z'}}\EE{\sup_{\tau\geq{\tau_q}}\prod_{i\in Z'}{\II{N_{i,\tau-1}=k_i} \II{g_i\pa{\abs{\bar\mu_{i,\tau-1}-\mu_i^*}}\geq u_i}}} \mathrm{d}\bu
    \nonumber\\&=
    \int_{\bu\in \R_+^{Z'}}\EE{\prod_{i\in Z'}{ \II{N_{i,\tau^*-1}=k_i}\II{g_i\pa{\abs{\bar\mu_{i,\tau^*-1}-\mu_i^*}}\geq u_i}}} \mathrm{d}\bu
    \label{rel:stodom},
\end{align}
where $\tau^*$ is the first $\tau\geq \tau_q$ such that $\II{\forall i \in Z',~N_{i,\tau-1}=k_i~\text{and}~g_i\pa{\abs{\bar\mu_{i,\tau-1}-\mu_i^*}}\geq u_i}$ holds, and is $\infty$ if it never holds. 
\begin{align*}
    \eqref{rel:stodom}&= \int_{\bu\in \R_+^{Z'}}\EE{\prod_{i\in Z'}{ \II{N_{i,\tau^*-1}=k_i}\II{g_i\pa{\abs{\bar\mu_{i,\tau^*-1}-\mu_i^*}}\geq u_i\vee g_i(0)}}} \mathrm{d}\bu
    \\&= \int_{\bu\in \R_+^{Z'}}\EE{\prod_{i\in Z'}{ \II{N_{i,\tau^*-1}=k_i}\II{{\abs{\bar\mu_{i,\tau^*-1}-\mu_i^*}}\geq g_i^{-1}\pa{u_i\vee g_i(0)}}}} \mathrm{d}\bu
    \\&= \int_{\bu\in \R_+^{Z'}}\sum_{\bs \in \set{-1,1}^{Z'}}\underbrace{\EE{\prod_{i\in Z'}{ \II{N_{i,\tau^*-1}=k_i}\II{{s_i\pa{\bar\mu_{i,\tau^*-1}-\mu_i^*}}\geq g_i^{-1}\pa{u_i\vee g_i(0)}}}}}_{\numterm{rel:subgauetstodom}}\mathrm{d}\bu
    \end{align*}

\begin{align*}
\eqref{rel:subgauetstodom}    
    &\leq {\PP{\frac{e^{\sum_{i\in Z'}{ \!N_{i,\tau^*-1}\!  \pa{4{s_ig_i^{-1}\pa{u_i\vee g_i(0)}}\pa{\bar\mu_{i,\tau^*-1}\!-\!\mu_i^*}-2{\pa{g_i^{-1}\pa{u_i\vee g_i(0)}}^2}}}}}{e^{\sum_{i\in Z'}{2\pa{g_i^{-1}\pa{u_i\vee g_i(0)}}^2k_i}{}}}\geq 1,\pa{N_{i,\tau^*-1}}_{i\in Z'}\!=\!\bk}} 
    \\&\leq
    {\PP{\frac{e^{\sum_{i\in Z'}{ N_{i,\tau^*-1}  \pa{4{s_ig_i^{-1}\pa{u_i\vee g_i(0)}}{}\pa{\bar\mu_{i,\tau^*-1}-\mu_i^*}-2{\pa{g_i^{-1}\pa{u_i\vee g_i(0)}}^2}{}}}}}{e^{\sum_{i\in Z'}2{\pa{g_i^{-1}\pa{u_i\vee g_i(0)}}^2k_i}{}}}\geq 1}} 
    \\&\leq
    {\frac{\EE{\exp\pa{\sum_{i\in Z'}{ N_{i,\tau^*-1}  \pa{4{s_ig_i^{-1}\pa{u_i\vee g_i(0)}}\pa{\bar\mu_{i,\tau^*-1}-\mu_i^*}-2{\pa{g_i^{-1}\pa{u_i\vee g_i(0)}}^2}{}}}}}}{\exp\pa{\sum_{i\in Z'}4{\pa{g_i^{-1}\pa{u_i\vee g_i(0)}}^2k_i}{}}}}
    \\&=
    {\frac{\EE{\exp\pa{\sum_{t=1}^{\tau^*-1}\sum_{i\in Z'\cap S_t}{   \pa{4{s_ig_i^{-1}\pa{u_i\vee g_i(0)}}{}\pa{X_{i,t}-\mu_i^*}-2{\pa{g_i^{-1}\pa{u_i\vee g_i(0)}}^2}{}}}}}}{\exp\pa{\sum_{i\in Z'}2{\pa{g_i^{-1}\pa{u_i\vee g_i(0)}}^2k_i}{}}}}.
\end{align*}
From Assumption~\ref{ass:subgau}, and from the fact that either $D_{\text{trig}}$ is independent from the outcomes, or the outcomes are mutually independent and each individual outcome  is independent from the fact that it is triggered, we have that \[M_{\tau}=\exp\pa{\sum_{t=1}^{\tau-1}\sum_{i\in Z'\cap S_t}{   \pa{4{s_ig_i^{-1}\pa{u_i\vee g_i(0)}}{}\pa{X_{i,t}-\mu_i^*}-2{\pa{g_i^{-1}\pa{u_i\vee g_i(0)}}^2}{}}}}\]
is a supermartingale:
\begin{align*}\EEc{M_{\tau}}{\cF_{\tau-1}}&=M_{\tau-1}\EEc{e^{{\sum_{i\in Z'\cap S_{\tau-1}}{   \pa{4{s_ig_i^{-1}\pa{u_i\vee g_i(0)}}{}\pa{X_{i,{\tau-1}}-\mu_i^*}-2{\pa{g_i^{-1}\pa{u_i\vee g_i(0)}}^2}{}}}}}}{\cF_{\tau-1}}\\&\leq M_{\tau-1}.\end{align*}
Since $\tau^*$ is a stopping time with respect to $\cF_\tau$, we have from Doob's optional sampling theorem for non-negative supermartingales\footnote{We use the version that relies on Fatou's lemma (\citet{durrett2019probability}, Theorem 5.7.6), so that it is not needed to have any additional condition on the stopping time $\tau^*$.} that $\EE{M_{\tau^*}}\leq 1$. Therefore, 
\begin{align*}
\eqref{rel:subgauetstodom}    
    &\leq \exp\pa{-\sum_{i\in Z'}2{\pa{g_i^{-1}\pa{u_i\vee g_i(0)}}^2k_i}{}}.
    \end{align*}
    Now, we want to use the following fact (see \citet{chang2011chernoff}):
    if $\eta\sim\cN(0,1)$, then with $\beta>1$,
    \[\sqrt{\frac{2e}{\pi}}\frac{\sqrt{\beta-1}}{\beta}e^{-\beta x^2/2}\leq \PP{\abs{\eta}\geq x}.\]
    Indeed, this gives
    \[\sqrt{\frac{2e}{\pi}}\frac{\sqrt{\beta-1}}{\beta}\exp\pa{-2{\pa{g_i^{-1}\pa{u_i\vee g_i(0)}}^2k_i}{}}\leq \PP{\abs{\eta_i}\geq  {g_i^{-1}\pa{u_i\vee g_i(0)}}\sqrt{\frac{4k_i}{\beta }}},\]
    where $\bheta\sim\cN(0,1)^{\otimes Z'}$.
    Thus, \begin{align*}
     \eqref{rel:stodom}&\leq \pa{\sqrt{\frac{\pi}{2e}}\frac{2\beta}{\sqrt{\beta-1}}}^{\abs{Z'}} \int_{\bu\in \R_+^{Z'}}\prod_{i\in Z'}\PP{\sqrt{\frac{\beta }{4k_i}}\abs{\eta_i}\geq  {g_i^{-1}\pa{u_i\vee g_i(0)}}}\mathrm{d}\bu\\&
     = \pa{\sqrt{\frac{\pi}{2e}}\frac{2\beta}{\sqrt{\beta-1}}}^{\abs{Z'}} \int_{\bu\in \R_+^{Z'}}\prod_{i\in Z'}\PP{g_i\pa{\sqrt{\frac{\beta }{4k_i}}\abs{\eta_i}}\geq  {{u_i\vee g_i(0)}}}\mathrm{d}\bu\\&
     = \pa{\sqrt{\frac{\pi}{2e}}\frac{2\beta}{\sqrt{\beta-1}}}^{\abs{Z'}} \int_{\bu\in \R_+^{Z'}}\prod_{i\in Z'}\PP{g_i\pa{\sqrt{\frac{\beta }{4k_i}}\abs{\eta_i}}\geq  {{u_i}}}\mathrm{d}\bu\\&
     = \pa{\sqrt{\frac{\pi}{2e}}\frac{2\beta}{\sqrt{\beta-1}}}^{\abs{Z'}} \prod_{i\in Z'}\int_{0}^\infty\PP{g_i\pa{\sqrt{\frac{\beta }{4k_i}}\abs{\eta_i}}\geq  {{u_i}}}\mathrm{d}u_i\\&
     = \pa{\sqrt{\frac{\pi}{2e}}\frac{2\beta}{\sqrt{\beta-1}}}^{\abs{Z'}} \prod_{i\in Z'}\EE{g_i\pa{\sqrt{\frac{\beta }{4k_i}}\abs{\eta_i}}}.
    \end{align*}
    We now want to bound $\EE{g_i\pa{\sqrt{\frac{\beta }{4k_i}}\abs{\eta_i}}}.$ We define $\alpha =2 - \sqrt{2}$, the unique solution in $(1/2,1)$ of $\alpha-1/2=(\alpha-1)^2/2$. Notice that $\alpha-1/2\geq 1/12$. Define $\eps_i\triangleq\eps\sqrt{\frac{4k_i}{\beta }}$.
    By definition, we have
\begin{align*}
    \EE{g_i\pa{\sqrt{\frac{\beta }{4k_i}}\abs{\eta_i}}} & = \int_{-\infty}^{+\infty} \frac{e^{-x^2/2}}{\int_{x - \varepsilon_i}^{x + \varepsilon_i} e^{-y^2/2} \mathrm{d}y}\mathrm{d}x - 1 \\
    & =  \underbrace{2\int_{\alpha \varepsilon_i}^{+ \infty}\frac{1}{\int_{x - \varepsilon_i}^{x + \varepsilon_i} e^{-\frac{y^2-x^2}{2}} \mathrm{d}y} \mathrm{d}x}_{A_1} + \underbrace{\int_{-\alpha \varepsilon_i}^{\alpha \varepsilon_i}\frac{e^{-x^2/2}}{\int_{x - \varepsilon_i}^{x + \varepsilon_i} e^{-y^2/2} \mathrm{d}y}\mathrm{d}x - 1}_{A_2}.
\end{align*}
We first bound $A_1$. With the change of variable $u = y-x$, we get:
\begin{align*}
    A_1 & = 2\int_{\alpha \varepsilon_i}^{+ \infty}\frac{1}{\int_{- \varepsilon_i}^{\varepsilon_i} e^{-u^2/2 - ux} \mathrm{d}u} \mathrm{d}x \\
    & \leq 2\int_{\alpha \varepsilon_i}^{+ \infty}\frac{1}{\int_{-\varepsilon_i}^{0} e^{-u^2/2 - ux} \mathrm{d}u} \mathrm{d}x
\end{align*}
Note that for $x\geq \alpha \varepsilon_i$ and $u \in [-\eps_i, 0]$, $-u^2/2 - ux \geq -(1-\frac{1}{2\alpha})ux$ and thus:
\begin{align}
    A_1 & \leq 2\int_{\alpha \varepsilon_i}^{+ \infty}\frac{1}{\int_{-\varepsilon_i}^{0} e^{-(1-\frac{1}{2\alpha})ux} \mathrm{d}u} \mathrm{d}x \nonumber\\
    & = 2\int_{\alpha \varepsilon_i}^{+ \infty} \frac{(1-\frac{1}{2\alpha})x}{e^{(1-\frac{1}{2\alpha})\varepsilon_i x}-1} \mathrm{d}x.\label{rel:tworegimes} \end{align}
    We distinguish two regimes. First, if $\eps_i^2\geq 12$, then
    \begin{align*}\eqref{rel:tworegimes}
    &\leq \frac{2e^{\pa{\alpha-\frac{1}{2}}\varepsilon_i^2}}{e^{(\alpha-\frac{1}{2})\varepsilon_i^2}-1}\int_{\alpha \varepsilon_i}^{+ \infty} \pa{1-\frac{1}{2\alpha}}xe^{-(1-\frac{1}{2\alpha})\varepsilon_i x} \mathrm{d}x \\
    & = \frac{2e^{(\alpha-\frac{1}{2})\varepsilon_i^2}}{e^{(\alpha-\frac{1}{2})\varepsilon_i^2}-1}\frac{1}{(1-\frac{1}{2\alpha})\varepsilon_i^2}\int_{(\alpha-\frac{1}{2})\varepsilon_i^2}^{+ \infty}xe^{- x} \mathrm{d}x \\
    & = \frac{2e^{(\alpha-\frac{1}{2})\varepsilon_i^2}}{e^{(\alpha-\frac{1}{2})\varepsilon_i^2}-1} \frac{1}{(1-\frac{1}{2\alpha})\varepsilon_i^2}\left[-(x+1)e^{-x} \right]_{(\alpha-\frac{1}{2})\varepsilon_i^2}^{\infty} \\
    & = \frac{2e^{(\alpha-\frac{1}{2})\varepsilon_i^2}}{e^{(\alpha-\frac{1}{2})\varepsilon_i^2}-1} \frac{1}{(1-\frac{1}{2\alpha})\varepsilon_i^2} \pa{\pa{\alpha-\frac{1}{2}}\varepsilon_i^2+1}e^{-(\alpha-\frac{1}{2})\varepsilon_i^2} \\
    & = \frac{2}{e^{(\alpha-\frac{1}{2})\varepsilon_i^2}-1}\pa{ \alpha+\frac{\alpha}{(\alpha-\frac{1}{2})\varepsilon_i^2}}
    \\
    & \leq {4 e^{-\varepsilon_i^2/12}} .
\end{align*}
Otherwise, we have
\begin{align*}
    \eqref{rel:tworegimes} &= \frac{2(1-\frac{1}{2\alpha})}{\varepsilon_i^2}\int_{\alpha \varepsilon_i^2}^\infty \frac{u}{e^{\pa{1-\frac{1}{2\alpha}}u}-1} \mathrm{d}u \\&\leq \frac{2(1-\frac{1}{2\alpha})}{\varepsilon_i^2}\int_{0}^\infty \frac{u}{e^{\pa{1-\frac{1}{2\alpha}}u}-1} \mathrm{d}u \\&= \frac{2(1-\frac{1}{2\alpha})}{\varepsilon_i^2}\frac{\pi^2}{6\pa{1-\frac{1}{2\alpha}}^2}
    \\&\leq \frac{6\beta}{\eps^2}.
\end{align*}

We now bound $A_2$. 
As $x \in [-\alpha \varepsilon_i, \alpha \varepsilon_i]$, it comes that $[-(1-\alpha)\varepsilon_i, (1-\alpha)\varepsilon_i] \subset [x-\varepsilon_i, x+\varepsilon_i]$. This implies that
\begin{align*}
A_2 & \leq \frac{\int_{-\alpha \varepsilon_i}^{\alpha \varepsilon_i} e^{-x^2/2} \mathrm{d} x }{\int_{-(1-\alpha) \varepsilon_i}^{(1-\alpha) \varepsilon_i} e^{-x^2/2} \mathrm{d} x} - 1\\
& = \frac{2 \int_{(1-\alpha)\varepsilon_i}^{\alpha \varepsilon_i} e^{-x^2/2}\mathrm{d}x}{\int_{-(1-\alpha) \varepsilon_i}^{(1-\alpha) \varepsilon_i} e^{-x^2/2} \mathrm{d} x}  \\
& \leq \frac{2 \int_{(1-\alpha)\varepsilon_i}^{\infty} e^{-x^2/2}\mathrm{d}x}{\int_{-(1-\alpha) \varepsilon_i}^{(1-\alpha) \varepsilon_i} e^{-x^2/2} \mathrm{d} x} 
\\
& \leq \frac{  e^{-{(1-\alpha)^2\varepsilon_i^2}/2}}{1-e^{-{(1-\alpha)^2\eps_i^2}/{2}}}\leq \pa{1+\frac{12}{\eps_i^2}}e^{-\eps_i^2/12}.
\end{align*}
The penultimate inequality relies on  $\int_{x}^\infty e^{-u^2/2} \mathrm{d}u {\leq} \sqrt{\frac{\pi}{2}}{e^{-x^2/2}}$ (see \citet{jacobs1965principles}, eq. (2.122)).
We obtain again two regimes: $2e^{-\eps_i^2/12}$ if $\eps_i^2\geq 12$, and $1+\frac{3\beta }{\eps^2}$ otherwise.
To summarize, we proved that $\eqref{rel:stodom}$ is bounded by
\[\pa{\sqrt{\frac{\pi}{2e}}\frac{2\beta}{\sqrt{\beta-1}}}^{\abs{Z'}} \prod_{i\in Z'} \pa{\II{\eps^2{\frac{4k_i}{\beta }}< 12}\pa{1+9\frac{\beta }{\eps^2}}
+\II{\eps^2{\frac{4k_i}{\beta }}\geq 12}{6e^{-\eps^2{\frac{k_i}{3\beta }}}}}.\]
After the summation on $\bk$, on $Z'$, on $q$, and on $Z$, we obtain that there exists two constants $C,C'$ such that
\[\sum_{Z\subset \textsc{t}(A^*),~Z\neq \emptyset}\sum_{q\geq 1}\sum_{Z'\subset Z,~Z'\neq \emptyset}\sum_{\bk\in [q..\infty)^{Z'}} \eqref{rel:stodom}\leq \pa{C\eps^{-2}\beta}\pa{\frac{C'\beta}{\sqrt{\beta-1}}\eps^{-4}\beta^2}^{m^*}.\]
 \end{proof}
\paragraph{Step 2: bound under $\fZ_t\wedge\fB_t$}
The filtered regret bound 
\[\EE{\sum_{t=1}^T\Delta_t\II{\fZ_t\wedge\fB_t}}\leq\frac{n\Delta_{\max}}{p^*}\pa{1+\pa{\frac{\Delta_{\min}}{2nB}-\frac{({m^*}^2+1)\eps}{n}}^{-2}}\]
is obtained as follows. Let $t\geq 1$. First, note that $\fB_t$ implies \[\set{\exists i \in \textsc{t}\pa{A_t} \text{ s.t. } nB_i\abs{\bar\mu_{i,t-1} - \mu^*_i}>{{\Delta_{\min}}/2-B({m^*}^2+1)\eps}}.\]
Then, fixing $i\in [n]$, we can ensure that $i\in S_t$ in the event, using that $p_i(A_t) = \PPc{i\in S_t}{\cF_t}$:
\begin{align*}&\EE{\sum_{t=1}^T \II{i\in \textsc{t}\pa{A_t},~nB_i\abs{\bar\mu_{i,t-1} - \mu^*_i}>{{\Delta_{\min}}/2-B({m^*}^2+1)\eps}}}\\&=\EE{\sum_{t=1}^T \frac{p_i(A_t)}{p_i(A_t)}\II{i\in \textsc{t}\pa{A_t},~nB_i\abs{\bar\mu_{i,t-1} - \mu^*_i}>{{\Delta_{\min}}/2-B({m^*}^2+1)\eps}}}\\&=\EE{\EEc{\sum_{t=1}^T \frac{1}{p_i(A_t)}\II{i\in {S_t},~nB_i\abs{\bar\mu_{i,t-1} - \mu^*_i}>{{\Delta_{\min}}/2-B({m^*}^2+1)\eps}}}{\cF_t}}
\\&=\EE{\sum_{t=1}^T \frac{1}{p_i(A_t)}\II{i\in {S_t},~nB_i\abs{\bar\mu_{i,t-1} - \mu^*_i}>{{\Delta_{\min}}/2-B({m^*}^2+1)\eps}}}
\\&\leq\EE{\sum_{t=1}^T \frac{1}{p^*}\II{i\in {S_t},~nB_i\abs{\bar\mu_{i,t-1} - \mu^*_i}>{{\Delta_{\min}}/2-B({m^*}^2+1)\eps}}}
\\&\leq\frac{1}{p^*}\sum_{t\geq 0} 1\wedge \pa{2\exp\pa{-2t\pa{\frac{\Delta_{\min}}{2nB_i}-\frac{B({m^*}^2+1)\eps}{nB_i}}^2}}
\\&\leq \frac{1}{p^*} \pa{1+\frac{2\exp\pa{-2\pa{\frac{\Delta_{\min}}{2nB_i}-\frac{B({m^*}^2+1)\eps}{nB_i}}^2}}{1-\exp\pa{-2\pa{\frac{\Delta_{\min}}{2nB_i}-\frac{B({m^*}^2+1)\eps}{nB_i}}^2}}}
\\&\leq\frac{1+\pa{\frac{\Delta_{\min}}{2nB_i}-\frac{B({m^*}^2+1)\eps}{nB_i}}^{-2}}{{p^*}}
\\&\leq\frac{1+\pa{\frac{\Delta_{\min}}{2nB}-\frac{({m^*}^2+1)\eps}{n}}^{-2}}{{p^*}}
.\end{align*}
\paragraph{Step 3: bound under $\fZ_t\wedge\fD_t$}
The filtered regret bound
\begin{align*} \EE{\sum_{t=1}^T\Delta_t\II{\fZ_t\wedge\fD_t}}\leq\Delta_{\max}\sum_{t\in [T]}\EE{\PPc{\fD_t}{\cH_t}}\leq \Delta_{\max}\sum_{t\in [T]}1/T=\Delta_{\max}\end{align*}
follows from the following Lemma~\ref{lem:nlem4}.
\begin{lemma}\label{lem:nlem4}
 In Algorithm~\ref{algo:tsbeta}~and~\ref{algo:tsgauss}, for all round $t\geq 1$, we have that the probability 
 \[\PPc{\norm{\bp\pa{A_t}\odot\bB\odot\pa{\btheta_{t}-\bar\bmu_{t-1}} }_1\geq\sqrt{{2}{\log\pa{2^n\pa{1+\ceil{\log_2\pa{{1}/{p^*}}}}^nT}}\sum_{i\in  [n]}\frac{p_i(A_t)^2B_i^2\beta}{N_{i,t-1}}}}{\cH_t}.\]
 is lower than $1/T$.
\end{lemma}
\begin{proof}
We rely on the fact that conditionally on the history, the sample $\btheta_t$ is either a Gaussian random vector of mean $\bar\bmu_{t-1}$ and of diagonal covariance given by $\beta N_{i,t-1}^{-1}/4$ (for Algorithm~\ref{algo:tsgauss}), or a product of Beta random variable, that is sub-Gaussian with the same covariance matrix \citep{marchal2017sub} (for Algorithm~\ref{algo:tsbeta}). We thus define the functions \[\alpha_t(A)\triangleq\sqrt{{2}{\log\pa{2^n\pa{1+\ceil{\log_2\pa{{1}/{p^*}}}}^nT}}\sum_{i\in  [n]}\frac{p_i(A_t)^2B_i^2\beta}{N_{i,t-1}}},\]\[\lambda_t(A) \triangleq \frac{2\alpha_t(A)}{\sum_{i\in A}\beta B_i^2p_i(A)^2/N_{i,t-1} }, \] 
we have, with $Q\triangleq\pa{\set{0}\cup\set{2^{-k},~k\in [\ceil{\log_2\pa{{1}/{p^*}}}]}}^n$, 
\begin{align*}&\PPc{\norm{\bp\pa{A_t}\odot\bB\odot\pa{\btheta_{t}-\bar\bmu_{t-1}} }_1\geq \alpha_t(\bp(A_t))}{\cH_t}\\
&\leq
\sum_{\bq\in Q}\PPc{\bq\leq\bp\pa{A_t}\leq 2\bq,\norm{\bp\pa{A_t}\odot\bB\odot\pa{\btheta_{t}-\bar\bmu_{t-1}} }_1\geq \alpha_t(\bp(A_t))}{\cH_t}
\\&\leq 
\sum_{\bq\in Q}\PPc{\norm{\bq\odot\bB\odot\pa{\btheta_{t}-\bar\bmu_{t-1}} }_1\geq \alpha_t(\bq)/2}{\cH_t}
\\&\leq 
\sum_{\bq\in Q} e^{-\lambda_t(\bq) \alpha_t(\bq)/2}\EEc{e^{\lambda_t(\bq) \norm{\bq\odot\bB\odot\pa{\btheta_{t}-\bar\bmu_{t-1}} }_1}}{\cH_t}
\\&\leq 
\sum_{\bq\in Q} e^{-\lambda_t(\bq) \alpha_t(\bq)/2}\prod_{i\in [n]}\EEc{e^{\lambda_t(\bq) q_iB_i\abs{{\theta_{i,t}-\bar\mu_{i,t-1}} }}}{\cH_t}
\\&\leq 
\sum_{\bq\in Q}e^{-\lambda_t(\bq) \alpha_t(\bq)/2}\prod_{i\in [n]}\EEc{e^{\lambda_t(\bq) q_iB_i\pa{{\theta_{i,t}-\bar\mu_{i,t-1}} }}+e^{\lambda_t(\bq) q_iB_i\pa{{\bar\mu_{i,t-1} -\theta_{i,t}} }}}{\cH_t}
\\&\leq \sum_{\bq\in Q}2^{n}e^{-\lambda_t(\bq) \alpha_t(\bq)/2}{e^{\lambda_t(\bq)^2 {\sum_{i\in A}\beta B_i^2q_i^2/(8 N_{i,t-1}) }}}\leq 1/T.
\end{align*}
\end{proof}
\paragraph{Step 4: bound under $\fZ_t\wedge\fC_t\wedge\neg\fB_t\wedge\neg\fD_t$}
We get that $\EE{\sum_{t=1}^T\Delta_t\II{\fZ_t\wedge\fC_t\wedge\neg\fB_t\wedge\neg\fD_t}}$ is bounded by
\begin{align*} n \Delta_{\max} + \sum_{i\in [n]}\frac{8\pa{3+\log\pa{m}}\beta B^2_{i}{\log\pa{2^n\pa{1+\ceil{\log_2\pa{\frac{1}{p^*}}}}^nT}}}{\min_{A\in \cA,~p_i(A)>0,~\Delta(A)>0}\Delta(A)/p_i(A)}
\end{align*}
from the following derivations. Let $t\geq 1$. 
Under $\fZ_t\wedge\fC_t\wedge\neg\fB_t\wedge\neg\fD_t$,  we have
\begin{align*}
  \Delta_t &\leq \norm{\bp\pa{A_t}\odot\bB\odot\pa{\btheta_{t}-\bmu^*} }_1+B\pa{{m^*}^2+1}\eps  &\fC_t
   \\&\leq \norm{\bp\pa{A_t}\odot\bB\odot\pa{\btheta_{t}-\bar\bmu_{t-1}} }_1+\norm{\bp\pa{A_t}\odot\bB\odot\pa{\bar\bmu_{t-1}-\bmu^*} }_1+B\pa{{m^*}^2+1}\eps
\\&\leq \norm{\bp\pa{A_t}\odot\bB\odot\pa{\btheta_{t}-\bar\bmu_{t-1}} }_1+\Delta_{\min}/2-B\pa{{m^*}^2+1}\eps+B\pa{{m^*}^2+1}\eps&\neg\fB_t
\\&\leq \norm{\bp\pa{A_t}\odot\bB\odot\pa{\btheta_{t}-\bar\bmu_{t-1}} }_1+\Delta_t/2&\fZ_t
\\&\leq \sqrt{{2}{\log\pa{2^n\pa{1+\ceil{\log_2\pa{{1}/{p^*}}}}^nT}}\sum_{i\in {A_t}}\frac{p_i(A_t)^2B_i^2\beta}{N_{i,t-1}}} +\Delta_t/2.&\neg\fD_t
\end{align*}
 Thus, the following event holds 
\[\fA_t\triangleq\set{ \Delta_t\leq \sqrt{{4\log\pa{2^n\pa{1+\ceil{\log_2\pa{{1}/{p^*}}}}^nT}}\sum_{i\in {A_t}}\frac{p_i(A_t)^2B_i^2\beta}{N_{i,t-1}}}},\]
and we can apply Lemma~\ref{lem:log(m)} to get the bound
$$\EE{\sum_{t=1}^T \Delta_t\II{\fA_t}}\leq n \Delta_{\max} + \sum_{i\in [n]}\frac{\pa{24+8\log\pa{m}}\beta B^2_{i}\log\pa{2^n\pa{1+\ceil{\log_2\pa{{1}/{p^*}}}}^nT}}{\min_{A\in \cA,~p_i(A)>0,~\Delta(A)>0}\Delta(A)/p_i(A)}.$$
\begin{lemma}[Adapted from \citet{Wang2018}]\label{lem:log(m)}
 Let's fix the time horizon $T$. For all $i\in [n]$, let ${\beta_{i,T}}\in \R_+$.
For $t\geq 1$,
consider the event
$$\fA_t\triangleq\set{\Delta_t \leq \sqrt{\frac{1}{2}\sum_{i\in [n]}\frac{p_i(A_t)^2\beta_{i,T}}{ N_{i,t-1}}}}.$$
Let 
$$\delta_{i,\min}\triangleq\min_{A\in \cA,~p_i(A)>0,~\Delta(A)>0}\Delta(A)/p_i(A).$$
Then, we have 
$$\EE{\sum_{t=1}^T \Delta_t\II{\fA_t}}\leq n \Delta_{\max} + \sum_{i\in [n]}\frac{\pa{3+\log\pa{m}}\beta_{i,T}}{\delta_{i,\min}}.$$
\end{lemma}
 \begin{proof}
 We use the regret allocation method from \citet{Wang2018}. Specifically, we want to prove that for
any time step $t$ where $\fA_t$ holds, we have the following allocation of the regret to each arm $i\in [n]$:
 \begin{align}\Delta_t \leq \sum_{i\in [n]}g_i(N_{i,t-1}),\label{rel:allocation}\end{align}
 where the allocation functions are defined for all $i\in [n]$ as
 $$ g_i(t)\triangleq
 \II{t=0}\Delta_{\max}
 + 
 \II{0<t\leq L_{i,2}} {\frac{p_i(A_t)\beta^{1/2}_{i,T}}{t^{1/2}}}
  + 
 \II{L_{i,2}<t\leq L_{i,1}} {\frac{p_i(A_t)\beta_{i,T}}{t\delta_{i,\min}}},
 $$
 $$L_{i,1}\triangleq \frac{m\beta_{i,T}}{\delta^2_{i,\min}},\quad L_{i,2}\triangleq \frac{\beta_{i,T}}{\delta^2_{i,\min}}.$$
 Indeed, we can already see that such an allocation produces the bound we are looking for. Notably the following derivation uses, for $i\in [n]$, the equality $p_i(A_t)=\EEc{\II{N_{i,t}=N_{i,t-1}+1}}{\cF_t}$, so that for all function $f$,
 a sum of the form $\sum_{t\geq 1}p_i(A_t)f(N_{i,t-1})$ is equal in expectation to $\sum_{t\geq 0} f(t)$. 
 \begin{align*}
 &\EE{\sum_{t=1}^T\II{\Delta_t \leq \sum_{i\in [n]}g_i(N_{i,t-1})}\Delta_t}
 \\
 &\leq \EE{\sum_{t\in [T]}\sum_{i\in [n]} g_i(N_{i,t-1})}
 \\
 &\leq 
 n \Delta_{\max} + \EE{\sum_{i\in [n]} \sum_{t=1}^{L_{i,2}} \frac{\beta^{1/2}_{i,T}}{t^{1/2}}} +\EE{\sum_{i\in [n]} \sum_{t=L_{i,2}+1}^{L_{i,1}} \frac{\beta_{i,T}}{t \delta_{i,\min}}} 
 \\
 &\leq
 n \Delta_{\max} + \sum_{i\in [n]}2\sqrt{L_{i,2}\beta_{i,T}} + \sum_{i\in [n]}\pa{1+\log\pa{\frac{L_{i,1}}{L_{i,2}}}} \frac{\beta_{i,T}}{\delta_{i,\min}}
 \\
 &=
  n \Delta_{\max} + \sum_{i\in [n]}\frac{2\beta_{i,T}}{\delta_{i,\min}} + \sum_{i\in [n]}\pa{1+\log\pa{m}} \frac{\beta_{i,T}}{\delta_{i,\min}}.
 \end{align*}
Therefore, we prove the lemma if we show that $\fA_t$ implies \eqref{rel:allocation}.
Let $t\geq 1$ be such that $\fA_t$ holds. We can assume that $\Delta_t>0$ (otherwise, the inequality~\eqref{rel:allocation} is trivial). We first observe that
 \begin{align}\nonumber
 \sum_{i\in [n],~N_{i,t-1}>L_{i,1}}{\frac{p_i(A_t)^2\beta_{i,T}}{\Delta_t N_{i,t-1}}}&\leq
     \sum_{i\in [n],~N_{i,t-1}>L_{i,1}}{\frac{p_i(A_t)^2\beta_{i,T}}{\Delta_t L_{i,1}}} \\&= \nonumber
      \sum_{i\in [n],~N_{i,t-1}>L_{i,1}}\pa{\frac{\delta_{i,\min}}{\Delta_t/p_i(A_t)}}^2 \frac{\Delta_t}{m}
      \\&\leq\nonumber
      \sum_{i\in [n],~N_{i,t-1}>L_{i,1}} \frac{\Delta_t}{m}
      \\&\leq \Delta_t.\label{rel:N>Li1}
 \end{align}
 This will be useful to prove that the allocation on all arms $i$ such that $N_{i,t-1}>L_{i,1}$ can be $0$.
 Then we distinguish the following cases:
 \begin{itemize}
     \item  If there exists $i_0\in A_t$ such that $N_{i_0,t-1}=0$, then 
 $$\Delta_t\leq \Delta_{\max}=g_{i_0}(N_{i_0,t-1}) \leq \sum_{i\in [n]} g_i(N_{i,t-1}).$$
 \item If there exists $i_0\in A_t$ such that $0<N_{i_0,t-1}\leq {p_{i_0}\pa{A_t}^2\beta_{i_0,T}}/{\Delta^2_{t}}$, then $N_{i_0,t-1}\leq L_{i_0,2}$ and we have 
 $$\Delta_t\leq \frac{p_{i_0}\pa{A_t}\beta^{1/2}_{i_0,T}}{N^{1/2}_{i_0,t-1}}={g_{i_0}\pa{N_{i_0,t-1}}}\leq \sum_{i\in [n]} g_i(N_{i,t-1}).$$
 \item If for all $i\in [n]$, $N_{i,t-1}> {p_{i}\pa{A_t}^2\beta_{i,T}}/{\Delta^2_{t}}$, then, 
 \begin{align}\nonumber\sum_{i\in [n],~L_{i,2}\geq N_{i,t-1}}\frac{p_i(A_t)^2\beta_{i,T}}{\Delta_t N_{i,t-1}} &= \sum_{i\in [n],~L_{i,2}\geq N_{i,t-1}}g_i(N_{i,t-1})\frac{p_i(A_t) \beta^{1/2}_{i,T}}{\Delta_t N^{1/2}_{i,t-1}}
 \\&\leq\label{rel:Li2>N}
 \sum_{i\in [n],~L_{i,2}\geq N_{i,t-1}}g_i(N_{i,t-1}).
 \end{align}
 On the other hand, using the event $\fA_t$, we have
 \begin{align*}
     \Delta_t & \leq \sum_{i\in [n]}\frac{p_i(A_t)^2\beta_{i,T}}{2\Delta_t N_{i,t-1}}
     \\& 
         = \sum_{i\in [n],~L_{i,1}\geq N_{i,t-1}}\frac{p_i(A_t)^2\beta_{i,T}}{2\Delta_t N_{i,t-1}} 
         + 
         \sum_{i\in [n],~N_{i,t-1}>L_{i,1}}{\frac{p_i(A_t)^2\beta_{i,T}}{2\Delta_t N_{i,t-1}}}. 
 \end{align*}
 Now, using \eqref{rel:N>Li1}, we get
 \begin{align*}
     \Delta_t & \leq \sum_{i\in [n],~L_{i,1}\geq N_{i,t-1}}\frac{p_i(A_t)^2\beta_{i,T}}{2\Delta_t N_{i,t-1}} + \frac{\Delta_t}{2}.
 \end{align*}
 We can therefore end the proof in the following way, using \eqref{rel:Li2>N},
  \begin{align*}
     \Delta_t & \leq \sum_{i\in [n],~L_{i,1}\geq N_{i,t-1}}\frac{p_i(A_t)^2\beta_{i,T}}{\Delta_t N_{i,t-1}} 
     \\
     &=
     \sum_{i\in [n],~L_{i,2}\geq N_{i,t-1}}\frac{p_i(A_t)^2\beta_{i,T}}{\Delta_t N_{i,t-1}}
     +
     \sum_{i\in [n],~L_{i,1}\geq N_{i,t-1}>L_{i,2}}\frac{p_i(A_t)^2\beta_{i,T}}{\Delta_t N_{i,t-1}}
     \\
     &\leq 
     \sum_{i\in [n],~L_{i,2}\geq N_{i,t-1}}g_i(N_{i,t-1})
     +
     \sum_{i\in [n],~L_{i,1}\geq N_{i,t-1}>L_{i,2}}\frac{p_i(A_t)\beta_{i,T}}{\delta_{i,\min} N_{i,t-1}}
     \\&= 
     \sum_{i\in [n],~L_{i,2}\geq N_{i,t-1}}g_i(N_{i,t-1})
     +
     \sum_{i\in [n],~L_{i,1}\geq N_{i,t-1}>L_{i,2}}g_i(N_{i,t-1})
     \\&=\sum_{i\in [n]} g_i\pa{N_{i,t-1}}.
 \end{align*}
 \end{itemize}
 \end{proof}
 \section{Proof of Theorem~\ref{thm:tsapprox}}\label{app:tsapprox}
 As for Theorem~\ref{thm:tsexact}, in the following proof, we will treat both algorithms at the same time. Let's first remark that we can apply steps 1,2 and 3 from Theorem~\ref{thm:tsexact}, but taking $r_1$ instead of the true reward function. Indeed,
 letting $E_{j,t}\in \argmax_{E\in \cE_{j,t}}r_j(E,\btheta_t)$ and
 $E_{j,t}^*\in \argmax_{E\in \cE_{j,t}}r_j(E,\bmu^*)$ for all $t\geq 1, j\in [\ell]$, we see that the sub-policy playing $E_{1,t}$ at round $t$ is actually minimizing the regret with respect to $r_1$ using \textsc{cts} with an exact oracle. Since $r_1$ satisfies the assumptions required for Theorem~\ref{thm:tsexact}, we can apply the steps. 
Although we still get $\Delta_{\max}$ in these bounds (because the suffered regret remains $\Delta_t$), we notice that $\Delta_{\min}$ is replaced by the minimal gap with respect to $r_1$. To get around this issue, when applying steps 1,2 and 3, we place ourselves under the event that the gap with respect to $r_1$ 
 is greater than $\Delta_{\min}/(2\sum_{j\in[\ell]} c_j)$. We can thus replace the minimal gap with respect to $r_1$ by this quantity in the bounds.
 To summarize, we can either place ourselves under the events (for $r_1$) of step 4, obtaining in parallel $3$ constant terms, or place ourselves under the event that the gap with respect to $r_1$ 
 is lower than $\Delta_{\min}/(2\sum_{j\in[\ell]} c_j)$. Our goal now is to do the same for the other reward functions $r_j$. However, since $\cE_{j,t}$ can depend on $E_{1,t},\dots,E_{j-1,t}$, we define the following filtration \[\pa{\cG_0,\cG_1,\dots,\cG_{\ell-1}}=\pa{\cH_t,\sigma\pa{\cH_t,E_{1,t}}, \sigma\pa{\cH_t,E_{1,t},E_{2,t}},\dots,\sigma\pa{\cH_t,E_{1,t},\dots,E_{\ell-1,t}}}.\]
 Let us suppose that we have treated the $r_1,\dots,r_{j-1}$ cases, then, we have at our disposal a filtered approximation regret against events $\fY_{1,t},\dots,\fY_{j-1,t}$ that $r_1,\dots,r_{j-1}$ are either in the situation of the step 4 or such the corresponding gap
 is lower than $\Delta_{\min}/(2\sum_{j\in[\ell]} c_j)$, respectively.
We can write the filtered approximation regret in the following way, by conditioning the expectation with this filtration to get rid of the randomness carried by $\cE_{j,t}$:
 \begin{align*}
  \EE{\EEc{\sum_{t\in [T]}\Delta_t}{\cG_{j-1}}\II{\fY_{1,t},\dots,\fY_{j-1,t}}}.\end{align*}
 We can now apply the same procedure as described above with the reward function $r_j$, on the inner conditional expectation, thus obtaining 3 additional $T$-independent terms and the new filtered approximation regret
  \begin{align*}
  \EE{\EEc{\sum_{t\in [T]}\Delta_t \II{\fY_{j,t}}}{\cG_{j-1}}\II{\fY_{1,t},\dots,\fY_{j-1,t}}}=\EE{\sum_{t\in [T]}\Delta_t\II{\fY_{1,t},\dots,\fY_{j,t}}}.\end{align*}
 Therefore, in the end, we have $3j$ constant terms and a filtered approximation regret where all reward functions $r_j$ are in the situation of step 4 or with a corresponding gap lower than $\Delta_{\min}/(2\sum_{j\in[\ell]} c_j)$. 
 Now, we place ourselves under this event to derive the dominant term of the bound on the approximation regret. 
 We let $J$ be the indices $j$ such that $r_j$ are in the situation of step 4.
 The derivation of step 4 applied to a function $r_j$ for $j\in J$ gives
 \begin{align*}
     r_j\pa{E_{j,t}^*,\bmu^*}-r_j\pa{E_{j,t},\bmu^*} &\leq 2 \norm{\bp\pa{A_t}\odot\bB_j\odot\pa{\btheta_{t}-\bar\bmu_{t-1}} }_1.
 \end{align*}
 We can therefore use equation~\eqref{eq:ora2} to obtain
 \begin{align*}
     \Delta_t &\leq \sum_{j\in [\ell]}\pa{ r_j\pa{E_{j,t}^*,\bmu^*}-r_j\pa{E_{j,t},\bmu^*}}\cdot c_j
     \\ &\leq 2\sum_{j\in [\ell]}\pa{{ r_j\pa{E_{j,t}^*,\bmu^*}-r_j\pa{E_{j,t},\bmu^*}} - \Delta_{\min}/(2\sum_{j\in[\ell]} c_j)}\cdot c_j
          \\ &\leq 2\sum_{j\in J}\pa{{ r_j\pa{E_{j,t}^*,\bmu^*}-r_j\pa{E_{j,t},\bmu^*}} - \Delta_{\min}/(2\sum_{j\in[\ell]} c_j)}\cdot c_j
     \\&\leq 4\sum_{j\in J}c_j \norm{\bp\pa{A_t}\odot\bB_j\odot\pa{\btheta_{t}-\bar\bmu_{t-1}} }_1.
 \end{align*}
 From there, we can repeat the end of the proof of Theorem~\ref{thm:tsexact} with the weight in front of each arm $i$ being $\sum_{j\in [\ell]} B_{ij}c_j$.
\section{Travelling salesman problem (TSP)}\label{app:tsp}
Here we give another example of a problem that falls into the \textsc{reduce2exact} setting, namely the \emph{Travelling salesman problem} (TSP). We will see that unlike the examples mentioned before, here we need to modify the algorithm a bit
to fully fall within the \textsc{reduce2exact} setting.

Given a complete undirected weighted graph $G=(V,E)$ whose distances $\bmu^*\triangleq(d(u,v))_{(u,v)\in E}$ have to satisfy the triangle inequality, the goal is to find an Hamiltonian cycle $A\in \cA\triangleq \set{\set{(v_0,v_1),\dots,(v_{\abs{V}-1},v_{\abs{V}})}: \set{v_1,\dots,v_{\abs{V}}=v_0}=V}$ of minimum cost $\sum_{i\in [\abs{V}]}d(v_{i-1},v_i)$. We consider the following oracle from the \citet{christofides1976worst} algorithm ($\alpha=2/3$).
\begin{itemize}
    \item $\ora_1(\bmu):$ The algorithm of Christofides 
    combines
$\cE_1\triangleq\set{\text{spanning trees of } G}$ and $\cE_2\triangleq\set{\text{perfect matchings of the subgraph }G'\text{ of }G\text{ induced by the vertices of odd order in }E_1 }$, with $r_1$ being the weight of the spanning tree and $r_2$ the weight of the perfect matching.
\item $\ora_2\pa{E_1,E_2}:$ $\ora_2$ combines the edges of $E_1$ and $E_2$ to form a connected multigraph $\tilde G=\pa{\tilde V, \tilde E}$ in which all vertices have even degree (so it is Eulerian), forms an Eulerian circuit in this multigraph, and finally, outputs the Hamiltonian cycle obtained by skipping repeated vertices (shortcutting).
\end{itemize}
  Thanks to the triangle inequality, shortcutting does not increase the weight, so we have
\begin{align*}
    \be_A\transpose\bmu^*\leq r_1(E_1,\bmu^*) + r_2(E_2,\bmu^*).
\end{align*}
Let's now deal with $A^*$. Removing an edge from $A^*$ produces a spanning tree, so $\be_{A^*}\transpose\bmu^* \geq r_1(E_1^*,\bmu^*)$.
On the other hand, by the triangle inequality,
the weight of the optimal TSP solution for $G'$ is lower than $\be_{A^*}\transpose\bmu^*$ (visiting more nodes does not, in any case, reduce the total cost). Taking every second edge of this cycle (which is of even length since all graphs have an even number of vertices of odd degree) we obtain a matching that has a weight less than half the weight of the cycle (if this is not the case we can take the complementary), so $\be_{A^*}\transpose\bmu^*/2\geq r_2(E_2^*,\bmu^*)$.
To summarize, we have
\begin{align*}
    \frac{2}{3}\be_A\transpose\bmu^* - \be_{A^*}\transpose\bmu^* &\leq \frac{2}{3}\pa{ r_1(E_1,\bmu^*) - r_1(E_1^*,\bmu^*) + r_2(E_2,\bmu^*) - r_2(E_2^*,\bmu^*)}.
\end{align*}
From the above, we can see that all the criteria of \textsc{reduce2exact} are satisfied, except Assumption~\ref{ass:smooth}. Indeed, we assume that we receive feedback from $A$, while we would need feedback from the set $\tilde E$. We could probably have foreseen that the TSP would pose a difficulty in our assumption: indeed, among the many operations performed by the oracle to build the final solution, shortcutting is the one that does not imply an optimization, but rather makes the solution feasible (so it does not represent a sub-problem as we have defined it in this paper). In other words, if we allowed the tour to pass over the same vertex several times, then the TSP would belong to \textsc{reduce2exact} by skipping the shortcutting step. 
Yet, we note that a workaround is possible by by taking a closer look at the shortcutting step: in this step, we have an Eulerian circuit that we follow by skipping some edges. Even if the skipped edges are replaced by new ones such that the distance traveled decreases, it is precisely the absence of feedback on the skipped edges that poses an issue. Therefore, for a given edge of the Eulerian circuit, it would be helpful to be able to guarantee that some feedback is obtained on this edge, i.e., that it has a chance to belong to the final Hamiltonian cycle. The trick is then to notice that the shortcutting step depends on the edge from which we start. In particular, this edge is guaranteed to be in the final Hamiltonian cycle.
This choice of the first edge is generally presupposed to be arbitrary and has no influence on the guarantees obtained previously, but for us, it can be used to force an edge to produce a feedback. 
More precisely, to choose this first edge, we can use a uniform randomization on $\tilde E$, meaning that for each edge $e\in \tilde E$, the probability $q_e$ that $e$ belongs to $A$ is such that $q_e\geq \frac{1}{\abs{\tilde E}}.$ 
Notice that since $\tilde E$ is a multigraph edge set, for a "true" edge $e\in E$, we can have two edges $e_1,e_2\in \tilde E$ representing it and in that case $q_{e_1}+q_{e_2}=p_e(A)$.
We thus get our Assumption~\ref{reduce2exact}:
\begin{align*}
    \abs{r_1(E_1,\bmu) + r_2(E_2,\bmu) - \pa{r_1(E_1,\bmu') + r_2(E_2,\bmu')}} \leq 
    \sum_{e\in \tilde E}\abs{\mu_e - \mu'_e}
    &\leq \abs{\tilde E} \sum_{e\in \tilde E}q_e\abs{\mu_e - \mu'_e}
    \\& = \abs{\tilde E} \sum_{e\in E}p_e(A)\abs{\mu_e - \mu'_e}.
\end{align*}
\end{document}